\documentclass{bmvc2k}

\usepackage{times}
\usepackage{epsfig}
\usepackage{amsmath}
\usepackage{amssymb}
\usepackage{pifont}

\usepackage{multirow}
\usepackage{tabularx}
\usepackage{dsfont}
\usepackage{url}
\usepackage{subfigure}
\usepackage{amsthm}
\usepackage[export]{adjustbox}
\usepackage{algorithmic}

\newtheorem{theorem}{Theorem}[section]
\newtheorem{lemma}[theorem]{Lemma}

\theoremstyle{definition}

\theoremstyle{remark}

\title{Learning Non-Parametric Invariances from Data with Permanent Random Connectomes}

\addauthor{Dipan K. Pal}{dipanp@andrew.cmu.edu}{1}
\addauthor{Akshay Chawla}{akshaych@andrew.cmu.edu}{1}
\addauthor{Marios Savvides}{marioss@andrew.cmu.edu }{1}

\addinstitution{
Dept. Electrical and Computer Engg.\\
 Carnegie Mellon University\\
 Pittsburgh, PA, USA
}

\runninghead{Permanent Random Connectome Networks}{}

\begin{document}

\maketitle

\begin{abstract}
  Learning non-parametric invariances directly from data remains an important open problem. In this paper, we introduce a new architectural layer for convolutional networks which is capable of learning general invariances from data itself. This layer can learn invariance to non-parametric transformations and interestingly, motivates and incorporates permanent random connectomes, thereby being called Permanent Random Connectome Non-Parametric Transformation Networks (PRC-NPTN). PRC-NPTN networks are initialized with random connections (not just weights) which are a small subset of the connections in a fully connected convolution layer. Importantly, these connections in PRC-NPTNs once initialized remain permanent throughout training and testing.  Permanent random connectomes make these architectures loosely more biologically plausible than many other mainstream network architectures which require highly ordered structures. We motivate randomly initialized connections as a simple method to learn invariance from data itself while invoking invariance towards multiple nuisance transformations simultaneously. We find that these randomly initialized permanent connections have positive effects on generalization, outperform much larger ConvNet baselines and the recently proposed Non-Parametric Transformation Network (NPTN) on benchmarks such as augmented MNIST, ETH-80 and CIFAR10, that enforce learning invariances from the data itself. 
\end{abstract}

\section{Introduction}
\textbf{The Problem of Invoking Invariances.} Learning invariances to nuisance transformations in data has emerged to be a core problem in machine learning \cite{anselmi2013unsupervised,hadsell2006dimensionality,gens2014deep,jaderberg2015spatial,cohen2016group}. Moving towards real-world data of different modalities, it is a daunting task to theoretically model all nuisance transformations. Towards this goal, methods which \textit{learn} non-parametric invariances from the data itself without any change in architecture will be critical. However, before delving into methods which learn such invariances, it is important to study methods which incorporate \textit{known} invariances in data. An early method to incorporate the translation prior was the Convolutional Neural Network (ConvNet) \cite{lecun1998gradient}. Over the years, there have been efforts in investigating what other transformations result in useful hand-crafted priors in data such as rotation and scale \cite{dieleman2015rotation,teney2016learning,li2017deep,mallat2013deep,cohen2016group,henriques2016warped,cohen2016steerable}. It is important to note however that these methods ultimately were limited to hand-crafted invariances assumed to be useful for the task at hand.


\textbf{Motivation of this Study:}  In this study, we motivate and investigate one possible architecture that can learn invariances towards multiple transformations from data itself. At the heart of the architecture is the structure called the \textit{permanent random} connectome. This simply refers to a channel shuffling layer that uses a fixed shuffling schedule throughout the life of the network (including training and testing) resulting in a \textit{permanent} connectome. Importantly however, this shuffling indexing is chosen at random at the initialization of the network. Thereby leading to the layer being referred to as the permanent \textit{random} connectome. We find that layers utilizing the permanent random operation allow architectures to learn multiple invariances efficiently from data itself. Our motivation also loosely stems from observations regarding connectomes in the cortex.



\begin{figure}
    \begin{center}
            \includegraphics[width=0.95\columnwidth,valign=m]{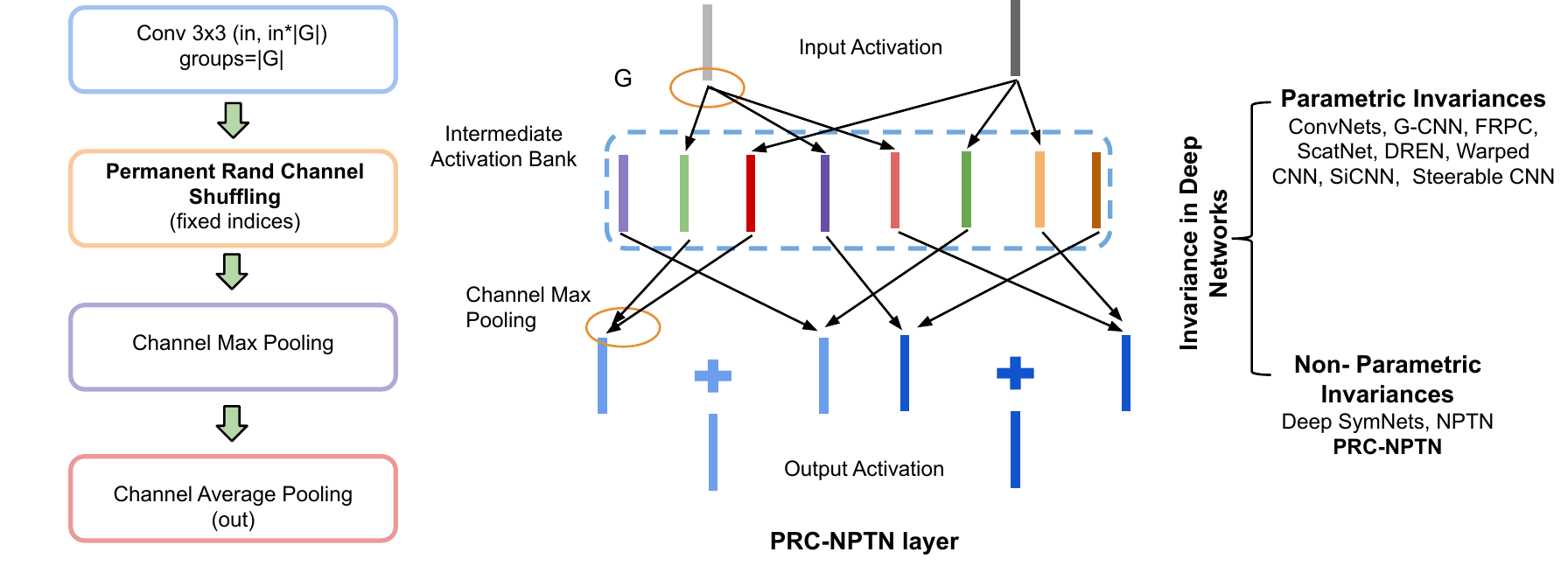}\label{fig_2_a}
    \end{center}
    \vspace{-0.5cm}
\caption{ \textbf{ Left:} Operations comprising the PRC-NPTN layer. The number of input and output channels in the Conv layer is (inch) and (inch$*|G|$) respectively. G is the number of filters (linear transformations learnt) for each input channel. The key operation proposed is the Permanent Random Channel shuffling operation with a fixed index mapping for every forward pass. This indexing or connectome is initialized randomly during network initialization. \textbf{ Center:} Architecture of the PRC-NPTN layer. Each input channel is convolved with a number of filters (parameterized by G). Each of the resultant activation maps is connected to a one of the channel max pooling units randomly (initialized once, fixed during training and testing). Each channel pooling unit pools over a fixed random support of a size parameterized by CMP. \textbf{ Right:} Explicit invariances enforced within deep networks in prior art are mostly parametric in nature. The important problem of learning \textit{non-parametric} invariances from data has not received a lot of attention.}
\label{fig_PRC-NPTN}
 \vspace{-0.5cm}

\end{figure}

\textbf{Prior Art Learning Invariances from Data or using Random Connectomes.} A different class of architectures that have been recently proposed explicitly attempt to \textit{learn} the transformation invariances directly from the data, with the only inductive bias being the \textit{structure or architecture} that allows them to do so. One of the earliest attempts using backpropagation was the SymNet \cite{gens2014deep}, which utilized kernel based interpolation to learn general invariances. Although given the interesting nature of the study, the method was limited in scalability. Spatial Transformer Networks \cite{jaderberg2015spatial} were also designed to learn activation normalization from data itself, however the transformation invariance learned was parametric in nature. A more recent effort was through the introduction of the Transformation Network paradigm \cite{pal2018nptn}. Non-Parametric Transformation Networks (NPTN)  were introduced as an generalization of the convolution layer to model general symmetries from data \cite{pal2018nptn}. It was also introduced as an alternate direction of network development other than skip connections. The convolution operation followed by pooling was re-framed as pooling across outputs from the translated versions of a filter. Translation forming a unitary group generates invariance through group symmetry \cite{anselmi2013unsupervised}. The NPTN framework has the important advantage of learning general invariances without any change in architecture while being scalable.  Given this is an important open problem, we introduce an extension of the Transformation Network (TN) paradigm with an enhanced ability to learn non-parametric invariances through permanent random connectivity.   There have been many seminal works that have indeed explored the role of \textit{temporary} random connections in deep networks such as DropOut \cite{srivastava2014dropout}, DropConnect \cite{wan2013regularization} and Stochastic Pooling \cite{zeiler2013stochastic}. However, unlike the proposed approach, the connections in these networks randomly change at \textit{every} forward pass, hence are temporary. More recently, random permanent connections were explored for large scale architectures \cite{xie2019exploring}. It is important however to note that the basic unit of computation, the convolutional layer, remained unchanged. Our study explores permanent random connectomes \textit{within} the convolutional layer itself, and explores how it can learn non-parametric invariances to multiple transformations simultaneously in a simple manner. We briefly visit other deep architectures that have been proposed over the years in the supplementary.


\begin{figure*}
    \begin{center}
        \subfigure[Homogeneous Structured Pooling]{%
        \centering
            \includegraphics[width=0.25\columnwidth,valign=m]{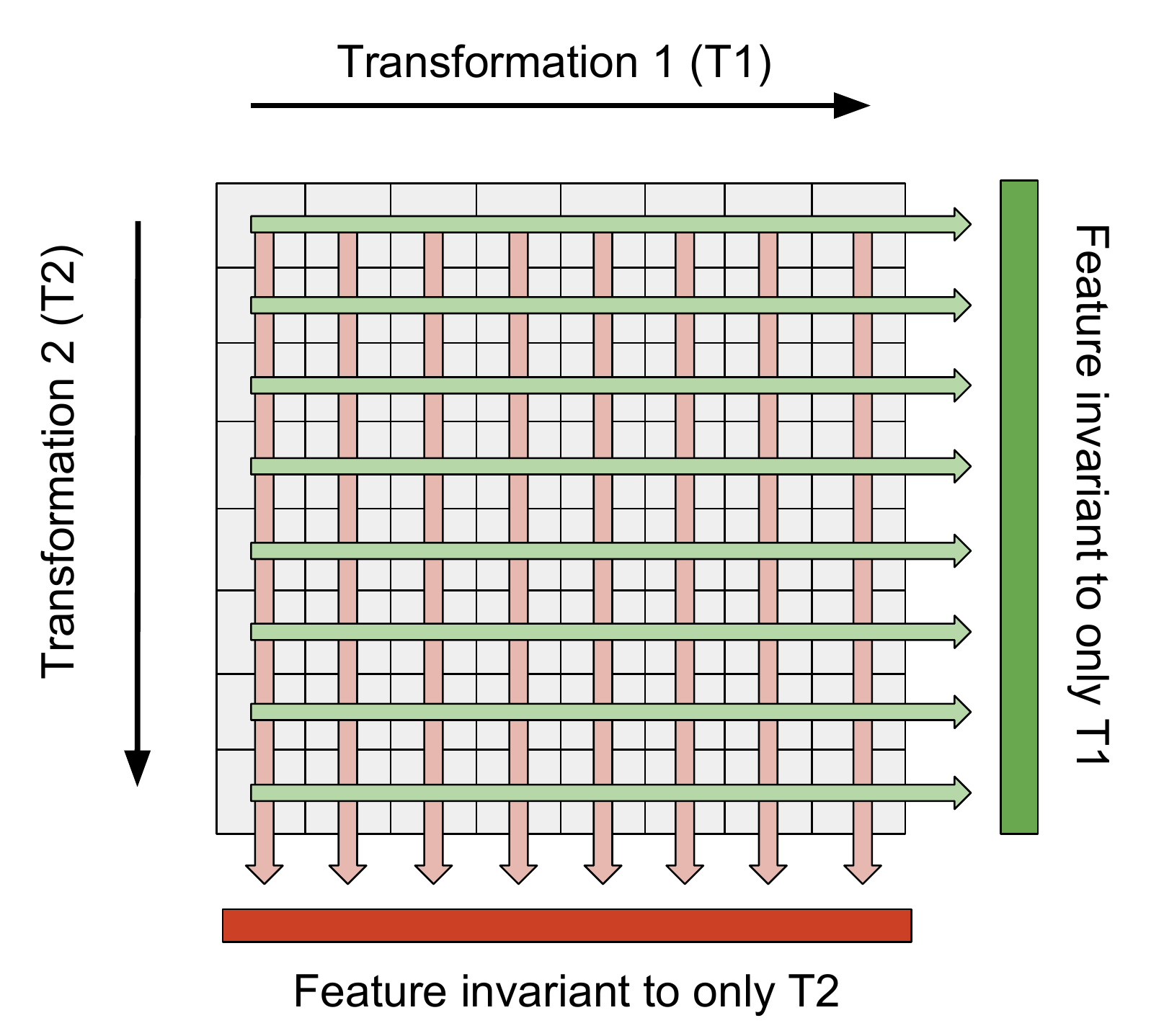}\label{fig_1_a}
        }%
         \subfigure[Permanent Random Support Pooling]{%
        \centering
        \includegraphics[width=0.25\columnwidth,valign=m]{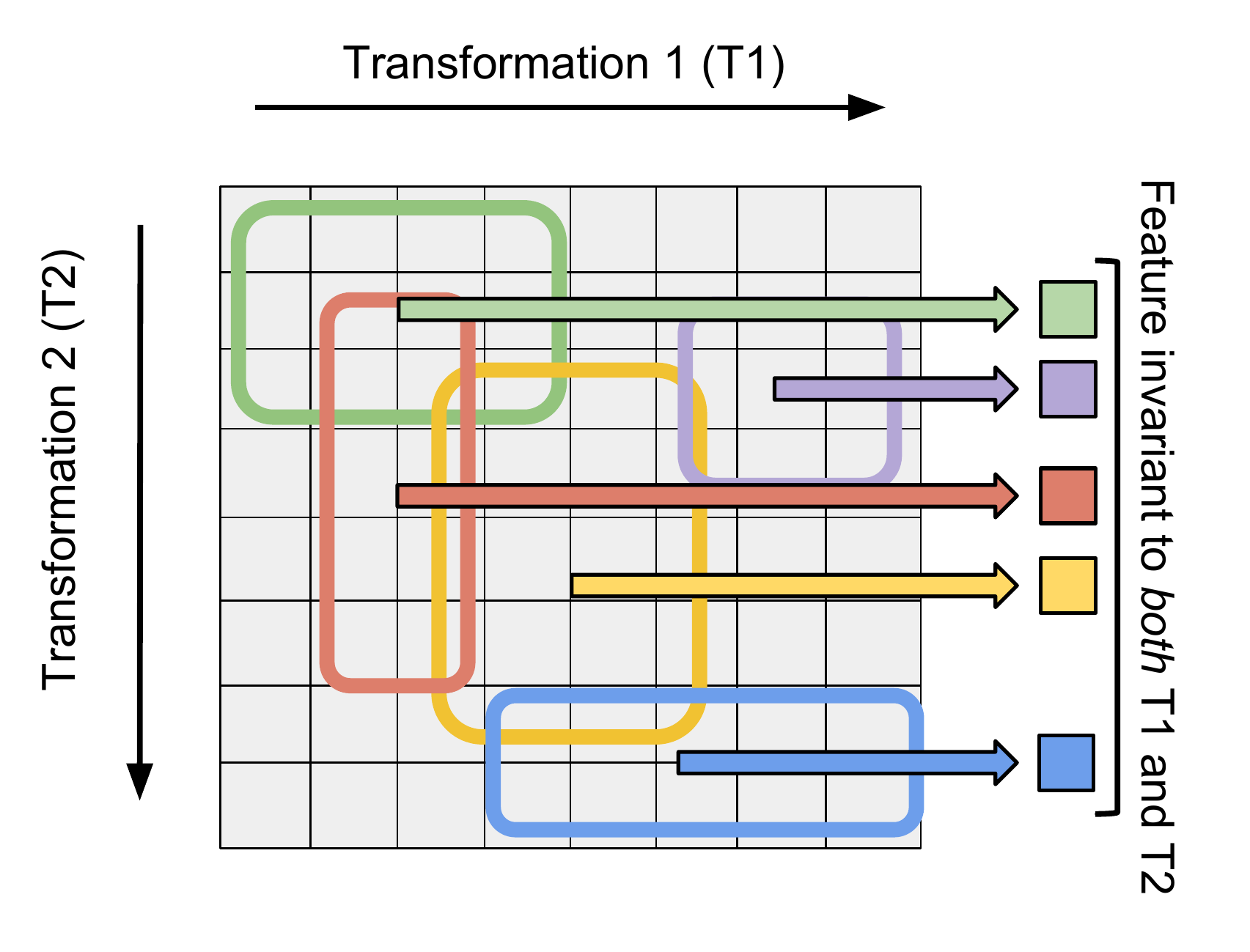}\label{fig_1_b}
        }
        \subfigure[PRC-NPTN Random Channel Pooling]{%
        \centering
         \includegraphics[width=0.4\columnwidth,valign=m]{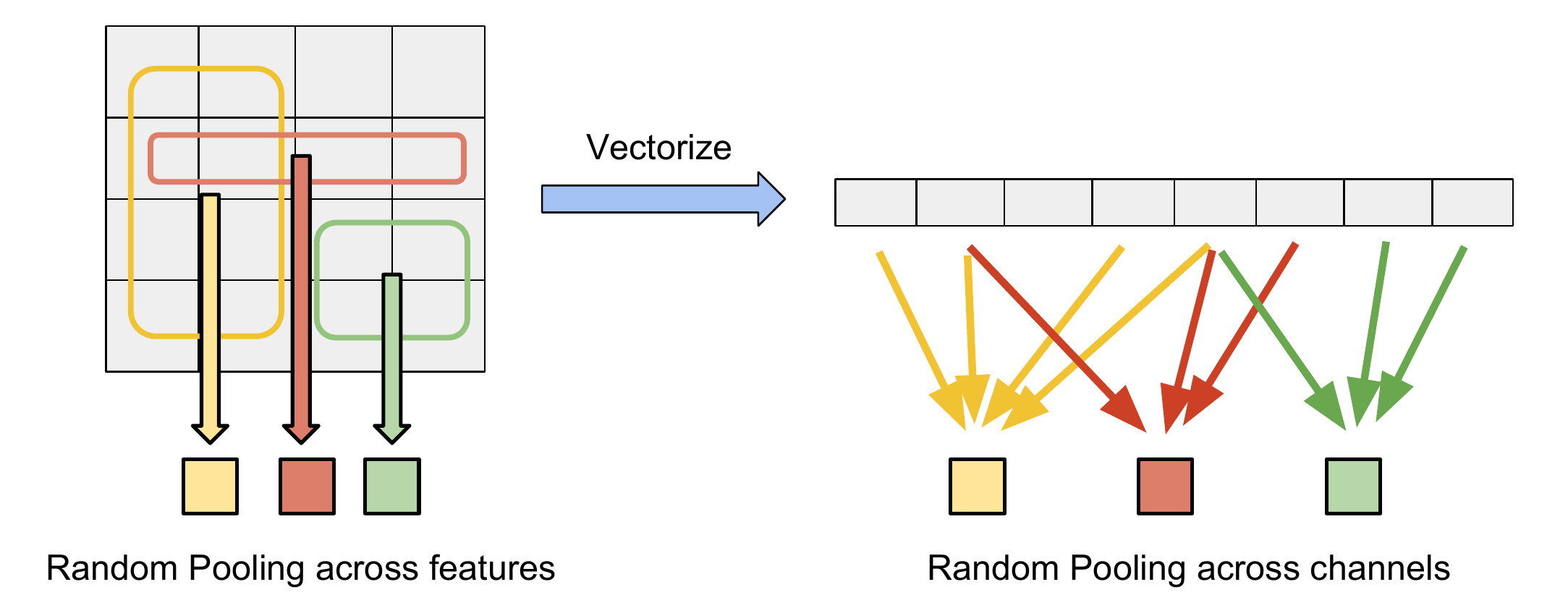}\label{fig_channelpooling}
        }

    \end{center}
    \vspace{-0.5cm}
\caption{ \textbf{(a) Homogeneous Structured Pooling} pools across the entire range of transformations of the \textit{same kind} leading to feature vectors invariant only to that particular transformations. Here, two \textit{distinct} feature vectors are invariant to transformation $T_1$ and $T_2$ independently. \textbf{(b) Heterogeneous Random Support Pooling} pools across randomly selected ranges of multiple transformations \textit{simultaneously}. The pooling supports are defined during initialization and remain fixed during training and testing. This results in a \textit{single} feature vector that is invariant to multiple transformations simultaneously. Here, each colored box defines the support of the pooling and pools across features only inside the boxed region leading to one single feature. \textbf{(c) Vectorized Random Support Pooling} extends this idea to convolutional networks, where one realizes that the random support pooling on the feature grid (on the left) is equivalent to random support pooling of the vectorized grid. Each element of the vector (on the right) now represents a single \textit{channel} in a convolutional network and hence random support pooling in PRC-NPTNs occur across channels.    }
\label{fig_conv_nptn}
    \vspace{-0.5cm}

\end{figure*}

\textbf{Relaxed Biological Motivation for Randomly Initialized Connectomes.} Although not central to our motivation, the observation that the cortex lacks precise \textit{local} pathways for back-propagation provided the initial inspiration for this study. It further garnered pull from the observation that random unstructured local connections are indeed common in many parts of the cortex \cite{corey2012cortical,schottdorf2015random}.  Moreover, it has been shown that orientation selectivity can arise in the visual cortex even through local random connections \cite{hansel2012mechanism}. Though we do not explore these biological connections in more detail, it is still an interesting observation. There has also been some interesting work which explored the use of random weight matrices for back propagation \cite{lillicrap2016random}. Here, the forward weight matrices were updated so as to fruitfully use the random weight matrices during back propagation. The motivation of the \cite{lillicrap2016random} study was to address the biological implausibility of the transport of precise gradients through the cortex due to the lack of exact connections and pathways \cite{grossberg1987competitive,stork1989backpropagation,mazzoni1991more,xie2003equivalence}.  The common presence of  random connections in the cortex at a \textit{local} level leads us to ask: Is it possible that such locally random connectomes improve generalization in deep networks? We provide evidence for answering this question in the positive.  

\textbf{Contributions.} 1) We motivate permanent random connectomes from the perspective of learning invariance to multiple transformations directly from data. The fundamental problem of learning non -parametric invariances in perception has not received enough attention. We present an architectural prior capable of such a task with loose biological motivation. 2) We present a theoretical result on learning invariances to transformations which do not obey a group structure in contrast to prior work. 3) We provide results on learning invariances to individual and multiple transformations in data without any change in architecture whatsoever. Further, we demonstrate improvements in generalization while using PRC-NPTN as a drop in replacement to conv layers in DenseNets. 4) Finally, as an engineering effort, we develop fast and efficient CUDA kernels for random channel pooling which result in efficient implementations of PRC-NPTNs in terms of computational speed and memory requirements compared to traditional Pytorch code.

\section{Permanent Random Connectome NPTNs}

 We begin by motivating permanent random connectomes from the perspective of selecting each specific support for pooling. We find that permanent random channel pooling invokes invariance to multiple transformations simultaneously. Investigating idea of pooling across transformations to invoke invariance, permanent random pooling emerges naturally. As part of our contribution, we present a theoretical result which confirms a long standing intuition that max pooling invokes invariance.
 
\textbf{Invoking Invariance through Max Pooling.} In previous years a number of theories have emerged on the mechanics of generating invariance through pooling. \cite{anselmi2013unsupervised,anselmi2017symmetry} develop a framework in which the transformations are modelled as a group comprised of unitary operators denoted by $\{ g \in \mathcal{G} \}  $. These operators transform a given filter $w$ through the operation $gw$\footnote{The action of the group element $g$ on $w$ is denoted by $gw$ to promote clarity.}, following which the dot-product between these transformed filters and an novel input $x$ is measured through $\langle x, gw  \rangle$. It was shown by \cite{anselmi2013unsupervised} that any moment such as the mean or max (infinite moment) of the distribution of these dot-products in the set $\{ \langle x, gw  \rangle  | g \in \mathcal{G}\}$ is an invariant. These invariants will exhibit robustness to the transformation in $\mathcal{G}$ encoded by the transformed filters in practice, as confirmed by \cite{liao2013learning,pal2016discriminative}. Though this framework did not make any assumptions on the distribution of the dot-products, it imposed the restricting assumption of group symmetry on the transformations. We now show that invariance can be invoked even when \textit{avoiding the assumption that the transformations in $\mathcal{G}$ need to form a group}. Nonetheless, we assume that the distribution of the dot-product $ \langle x, gw  \rangle $ is uniform and thus we have the following result\footnote{We provide a proof in the supplementary and thank Purvasha Chakravarti at CMU for the proof.. The assumption of the distribution being uniform is meant to provide insight into the general behavior of the max pooling operation, rather than a statement that deep learning features are uniformly distributed.}.

\begin{lemma}\label{lem_invariance} (Invariance Property) Assume a  novel test input $x$ and a  filter $ w$ both fixed vectors $ \in \mathbb{R}^d$. Further, let $g$ denote a random variable representing unitary operators with some distribution. Finally, let  $\Upsilon(x) =  \langle  x, gw  \rangle$, with $\Upsilon(x)   \sim  U(a,b) $   \emph{i.e.} a Uniform distribution between $a$ and $b$. Then, we have $$ \mathrm{Var}(\max \Upsilon(x)) \leq  \mathrm{Var}( \Upsilon (x) ) = \mathrm{Var} (\langle  g^{-1}x, w  \rangle) $$

\end{lemma}

This result is interesting because it shows that the max operation of the dot-products has less variance due to $g$ than the pre-pooled features. Though this is largely known empirical result, a concrete proof for invoking invariance was so far missing. More importantly, it bypasses the need for a group structure on the nuisance transformations $\mathcal{G}$. Practical studies such as \cite{liao2013learning,pal2016discriminative} had ignored the effects of non-group structure in theory while demonstrating effective empirical results. Also note that the variance of the max is less than the variance of the quantity $\langle g^{-1}x, w \rangle$, which implies that  $\max \Upsilon(x)$ is more robust to $g$ even in test, though it has never observed $gx$. This useful property is due to the unitarity of $g$.

\textbf{Connection to Deep Networks.} PRC-NPTN as we will show, perform max pooling across \textit{channels} not space, to invoke invariance. In the framework $\langle x, gw  \rangle$, $w$ would be one convolution filter with $g$ being the transformed version of it. Note that this modelling is done only to satisfy a theoretical construction, we do not actually transform filters in practice. All transformed filters are learnt through backpropagation. This framework is already utilized in ConvNets. For instance, ConvNets \cite{lecun1998gradient} pool only across translations (convolution operation itself followed by \textit{spatial} max pooling implies $g$ to be translation).


\textbf{Invoking Invariance through Channel Pooling in Deep Networks.}  Consider a grid of features that have been obtained through a dot product $\langle x, gw  \rangle$ (for instance from a convolution activation map, where the grid is simply populated with each $k\times k \times 1$ filter, not $k\times k \times c$) (see Fig.~\ref{fig_1_a}). Assume that along the two axes of the grid, two different kinds of transformation are acted. $T_1$ along the horizontal axis and $T_2$ along the vertical. $T_1=g_1(\cdot; \theta_1)$ where $g_1$ is a transformation parameterized by $\theta_1$ that acts on $w$ and similarly  $T_2=g_2(\cdot; \theta_2)$. Now, pooling homogeneously across one axis invokes invariance \textit{only} to the corresponding $g$ (for a more in depth analysis see \cite{anselmi2013unsupervised}). Similarly, pooling along $T_2$ only will result in a feature vector (Feature 2)  invariant \textit{only} to $T_2$. These representations (Feature 1 and 2) have complimentary invariances and can be used for complimentary tasks \emph{e.g.}  face recognition (invariant to pose) versus pose estimation (invariant to subject). This approach has one major limitation that this scales linearly with the number of transformations which is impractical. One therefore would need features that are invariant to multiple transformations simultaneously. A simple yet effective approach is to pool along all axes thereby being invariant to all transformations simultaneously. However, doing so will result in a degenerative feature (that is invariant to everything and discriminative to nothing). Therefore, the key is to limit the \textit{range} of pooling performed for each transformation.






\textbf{Choosing the Support for Pooling at Random: Permanent Random Connectomes.} A solution to trivial feature problem described above, is to limit the \textit{range} or support of pooling as illustrated in Fig.\ref{fig_1_b}. One simple way of selecting such a support for pooling is at random. This selection would happen only once during initialization of the network (or any other model), and will remain fixed through training and testing. In order to increase the selectivity of such features, multiple such pooling units are needed with such a randomly initialized support \cite{anselmi2013unsupervised,pal2017mmif}. These multiple pooling units together form the feature that is invariant to multiple transformations simultaneously, which improves generalization as we find in our experiments. This is called heterogeneous pooling and Fig.~\ref{fig_1_b} illustrates this more concretely. We therefore find that permanent random pooling is motivated naturally through the need to attain invariance to multiple transformations simultaneously.

\textbf{The PRC-NPTN layer.}   Fig.~\ref{fig_PRC-NPTN} shows the the architecture of a single PRC-NPTN layer \footnote{We provide pseudo-code in the supplementary}. The PRC-NPTN layer consists of a set of $N_{in}\times G$  filters of size $k\times k$ where $N_{in}$ is the number of input channels and $G$ is the number of filters connected to each input channel. More specifically, each of the $N_{in}$ input channels connects to $|G|$ filters.  Then, a number of channel max pooling units randomly select a fixed number of activation maps to pool over. This is parameterized by Channel Max Pool (CMP). Note that this random support selection for pooling is the reason a PRC-NPTN layer contains a permanent random connectome. These pooling supports once initialized do not change through training or testing. Once max pooling over CMP activation maps completes, the resultant tensor is average pooled across channels with a average pool size such that the desired number of outputs is obtained. After the CMP units, the output is finally fed through a two layered network with the same number of channels with $1\times 1$ kernels, which we call a pooling network. This small pooling network helps in selecting non-linear combinations of the invariant nodes generated through the CMP operation, thereby enriching feature combinations downstream. For experimental rigor, we also benchmark against the baseline ConvNet supplemented with this 1x1 pooling network.


\begin{figure}
    \begin{center}
         \subfigure[Rotation $0^\circ$]{%
        \centering
        \includegraphics[width=0.23\columnwidth,valign=m]{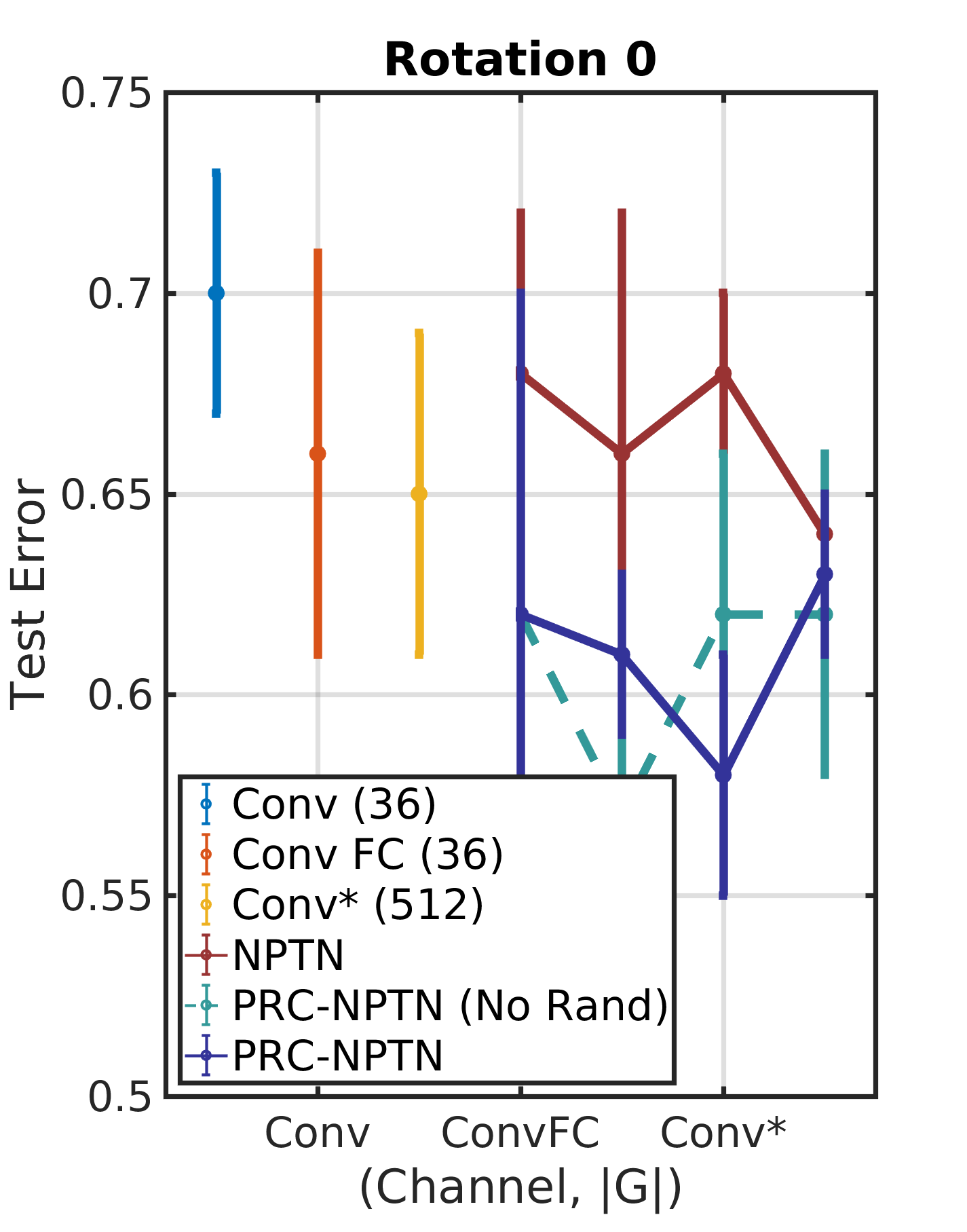}
        }
        \subfigure[Rotation $30^\circ$]{%
        \centering
            \includegraphics[width=0.23\columnwidth,valign=m]{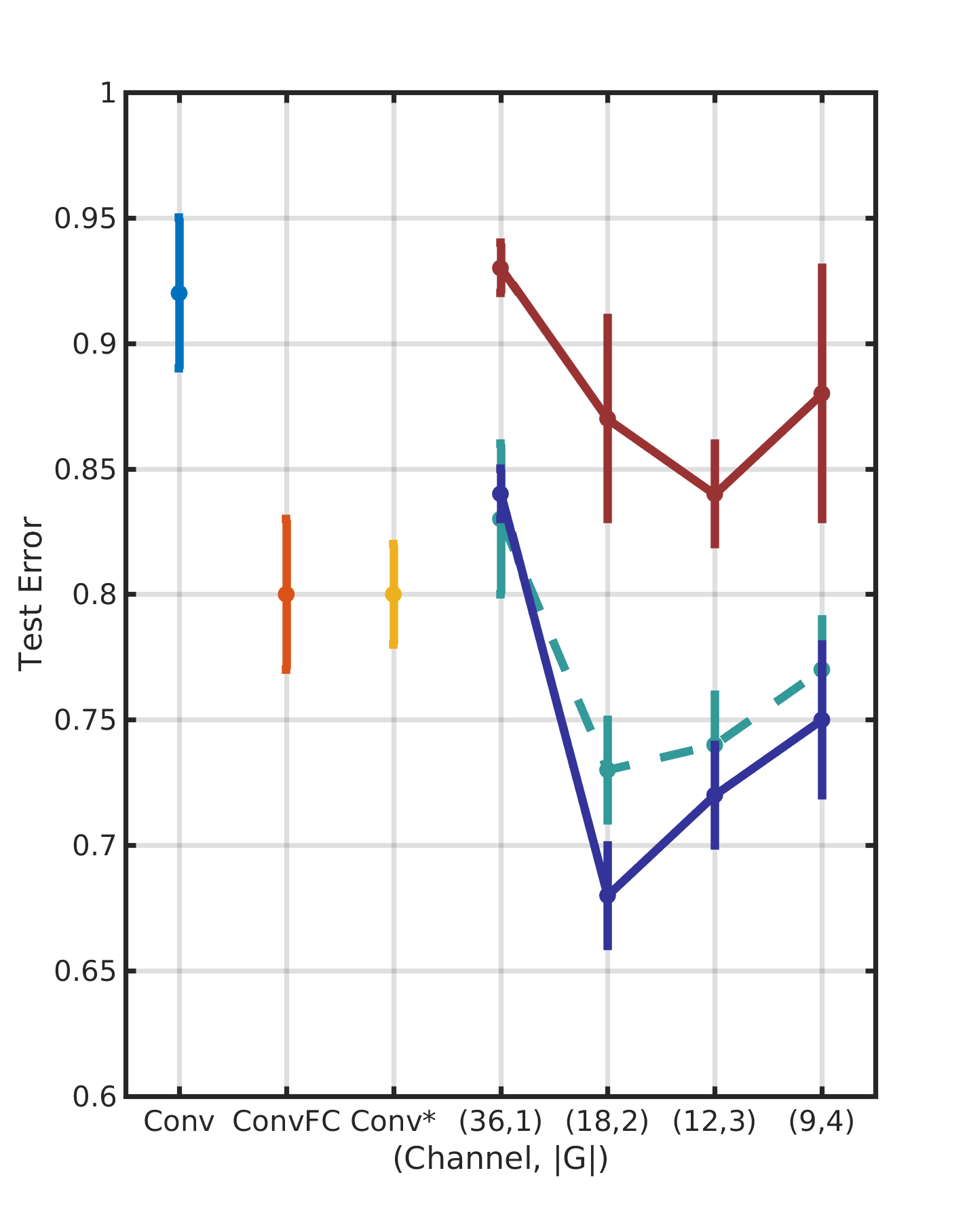}
        }
        \subfigure[Rotation $60^\circ$]{%
        \centering
         \includegraphics[width=0.23\columnwidth,valign=m]{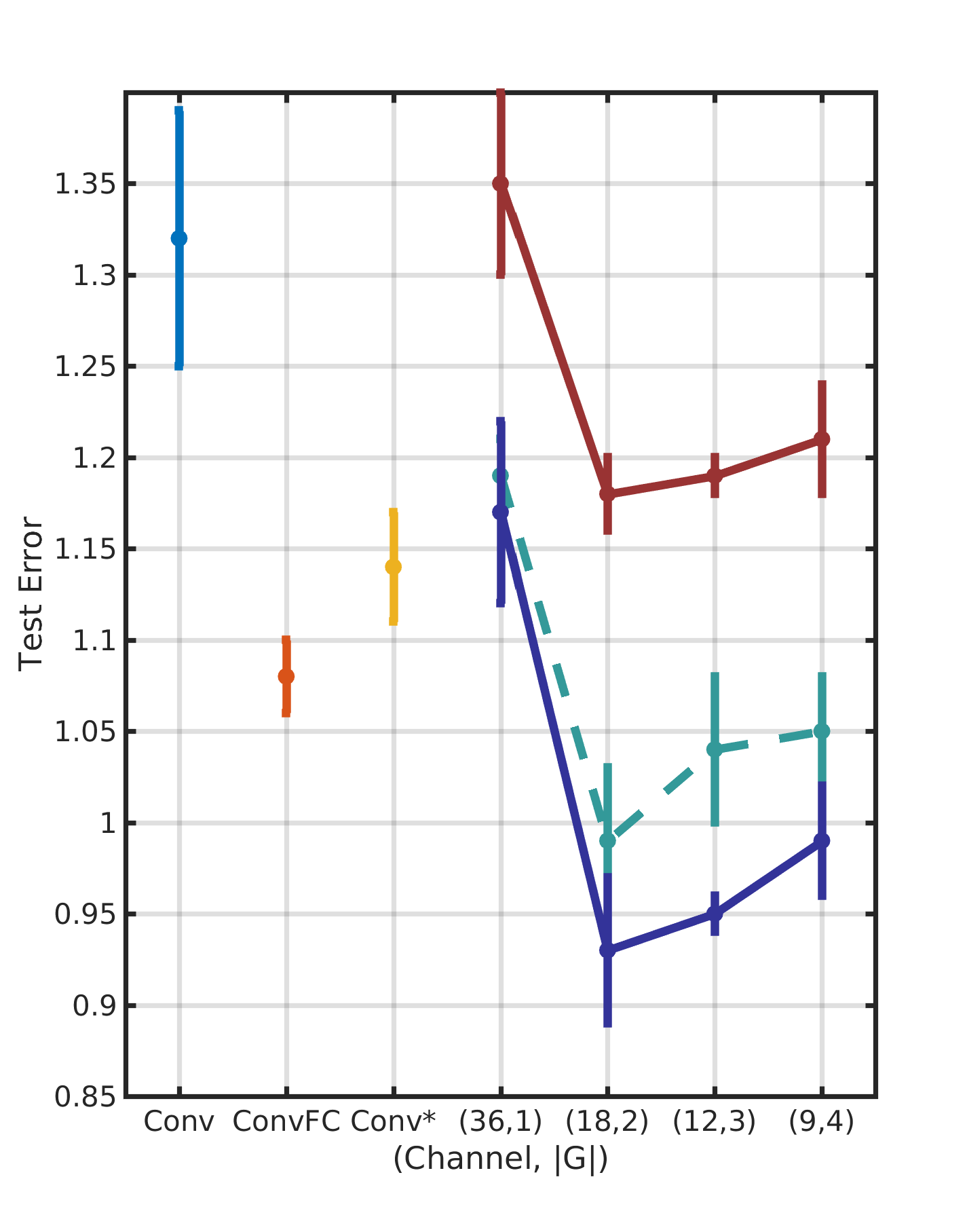}
        }
      \subfigure[Rotation $90^\circ$]{%
        \centering
         \includegraphics[width=0.23\columnwidth,valign=m]{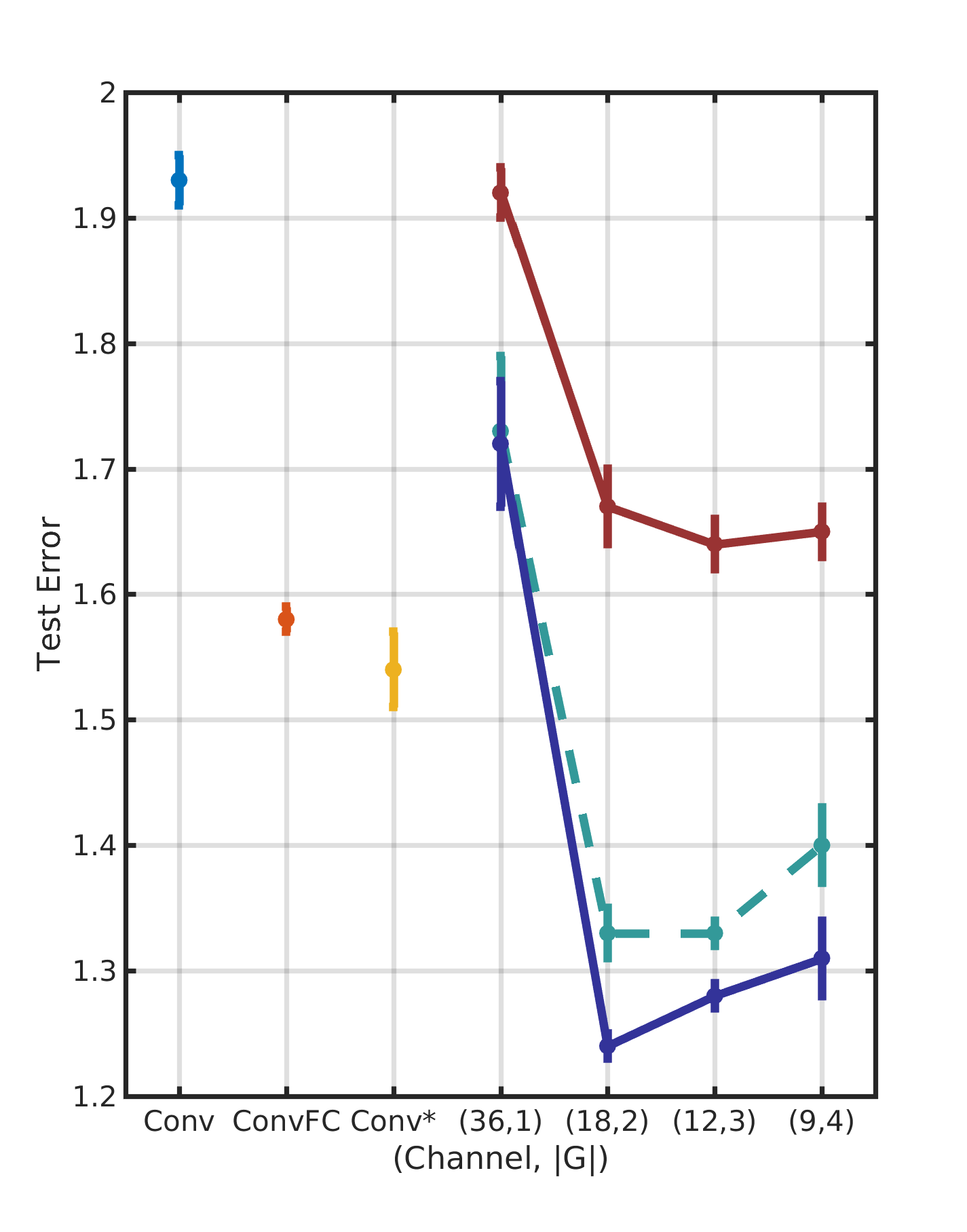}
        }      

    \end{center}
    \vspace{-0.5cm}
\caption{ \textbf{Individual Transformation Results: }Test error statistics with mean and standard deviation on  MNIST with progressively extreme transformations with a) \textbf{random rotations} and b) \textbf{random pixel shifts} (in supplementary). Only for PRC-NPTN and NPTN the brackets indicate the number of channels in the layer 1 and $G$. ConvNet FC denotes the addition of a 2-layered pooling $1\times 1$ pooling network after every layer. Note that for this experiment, CMP=$|G|$. Permanent Random Connectomes help with achieving better generalization despite increased nuisance transformations. We provide experiments on CIFAR 10 using DenseNets in the supplementary. }
\label{fig_minist_indi}

\end{figure}




\section{Empirical Evaluation and Discussion}

\textbf{Goal.} The goal of our evaluation study is to demonstrate PRC-NPTNs as capable of learning transformations from data and  to showcase improvements in generalization in supervised classification over relevant baselines. The goal is not in fact, to compete with the state of the art approaches for any dataset.

\textbf{General Experimental Settings. }  For all experiments, we run all models for 300 epochs trained using SGD. The initial learning rate was kept at 0.1 and decreased by 10 at 50\% and 75\% epoch completion. Momentum was kept at 0.9 with a weight decay of $10^{-5}$. Batch size was kept at 64 for both MNIST and ETH-80 \footnote{We provide experiments on CIFAR-10 in the supplementary.}. For the MNIST experiments, gradients were clipped to norm 1. Each block for all baselines for ConvNet and PRC-NPTN had either a convolution layer or PRC-NPTN layer followed by batch normalization, PReLU and spatial max pooling. The convolutional kernel size for all models was kept at $5 \times 5$ for all MNIST experiments and $3 \times 3$ for all other models. Spatial max pooling of size $3\times 3$ was performed after every layer, BN and PReLU for MNIST models. 

\textbf{Limitations in Typical Implementations and Developing Faster Kernels.} Our implementation with traditional PyTorch still suffered from heavy GPU memory use and slower run times despite optimizing code at the PyTorch abstraction level. The key bottleneck in computational and memory efficiency was found to be the randomized channel pooling operation. The issue was addressed by developing CUDA kernels that performed pooling on non-contiguous blocks of memory without creating copies of the same. This allowed for faster non-contiguous pooling over feature and activation maps with a significant reduction in memory usage. The operation was built as a CUDA-kernel that is interfaced with PyTorch through CuPy. This engineering effort is part of our contribution and we demonstrate significant improvements in memory and computational efficiency in our experiments. We show improvements in memory and speed as different aspects of the PRC-NPTN networks are changed in the supplementary \footnote{We provide details of the architecture, benchmarking  and additional experiments in the supplementary.}. We observe a consistent speed up of atleast 1.5x and a significant reduction in memory usage.





\begin{figure}
    \begin{center}
         \subfigure[($0^\circ$,  $0$ pix)]{%
        \centering
        \includegraphics[width=0.23\columnwidth,valign=m]{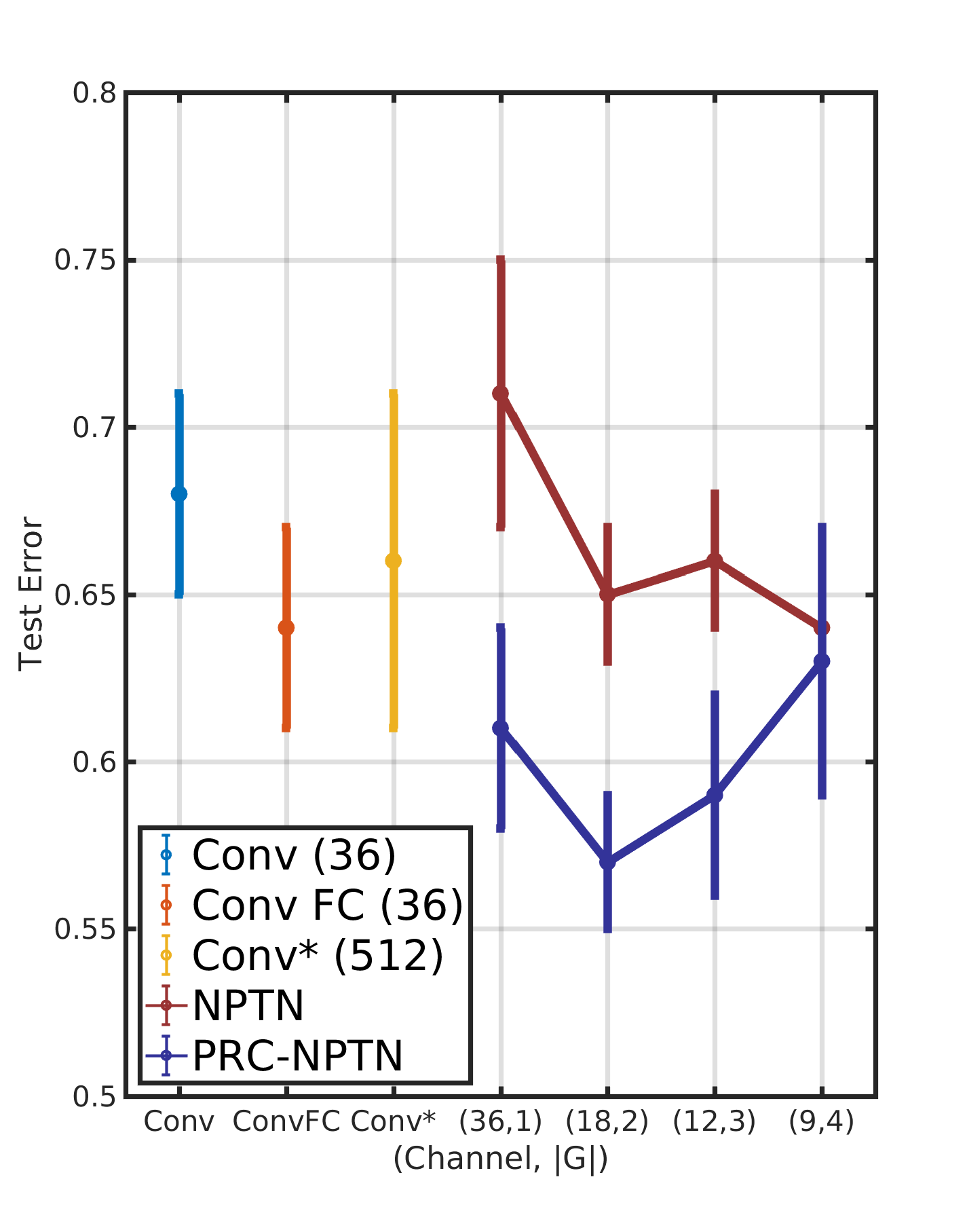}
        }
        \subfigure[($30^\circ$, $4$ pix)]{%
        \centering
            \includegraphics[width=0.23\columnwidth,valign=m]{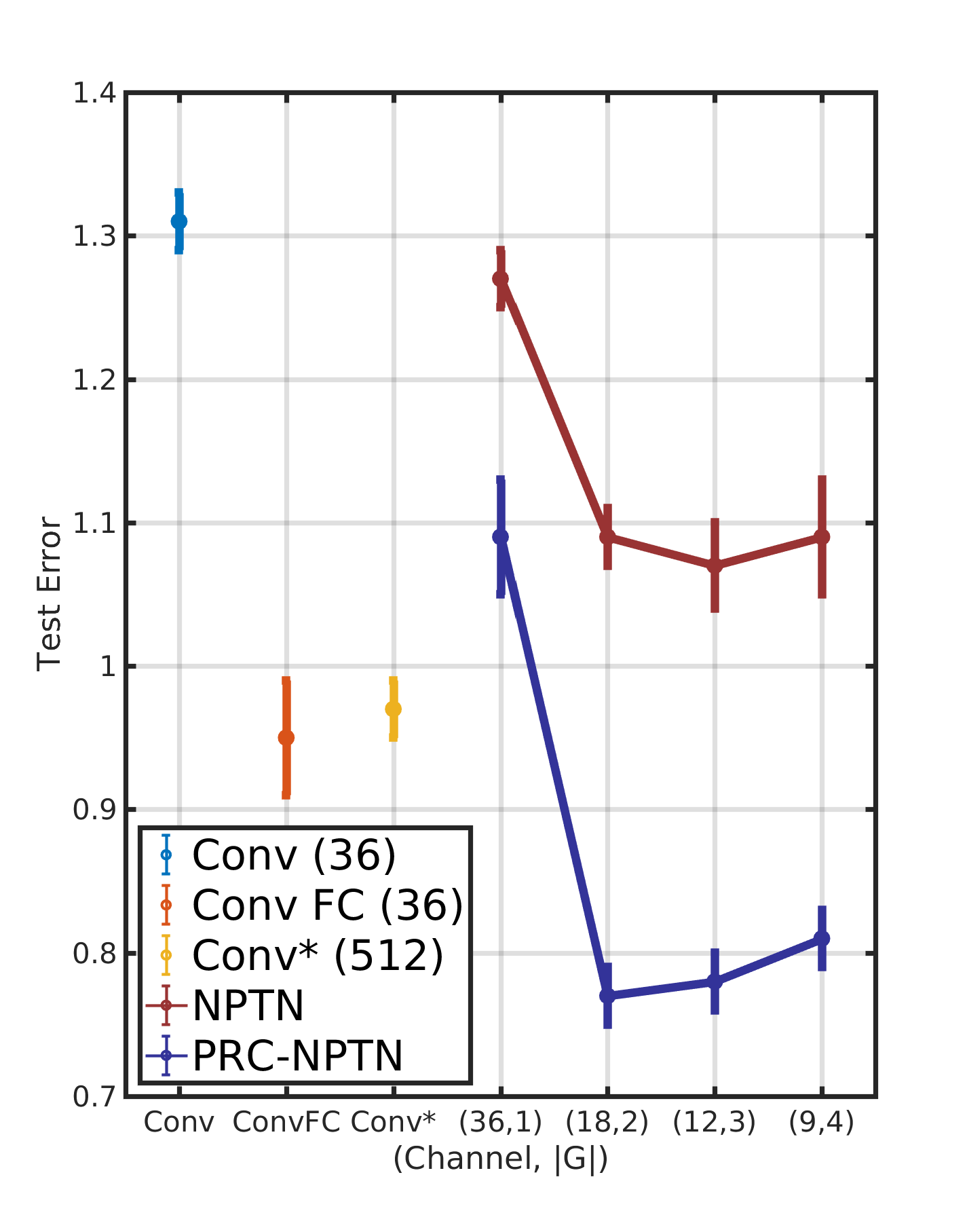}
        }
        \subfigure[($60^\circ$,  $8$ pix)]{%
        \centering
         \includegraphics[width=0.23\columnwidth,valign=m]{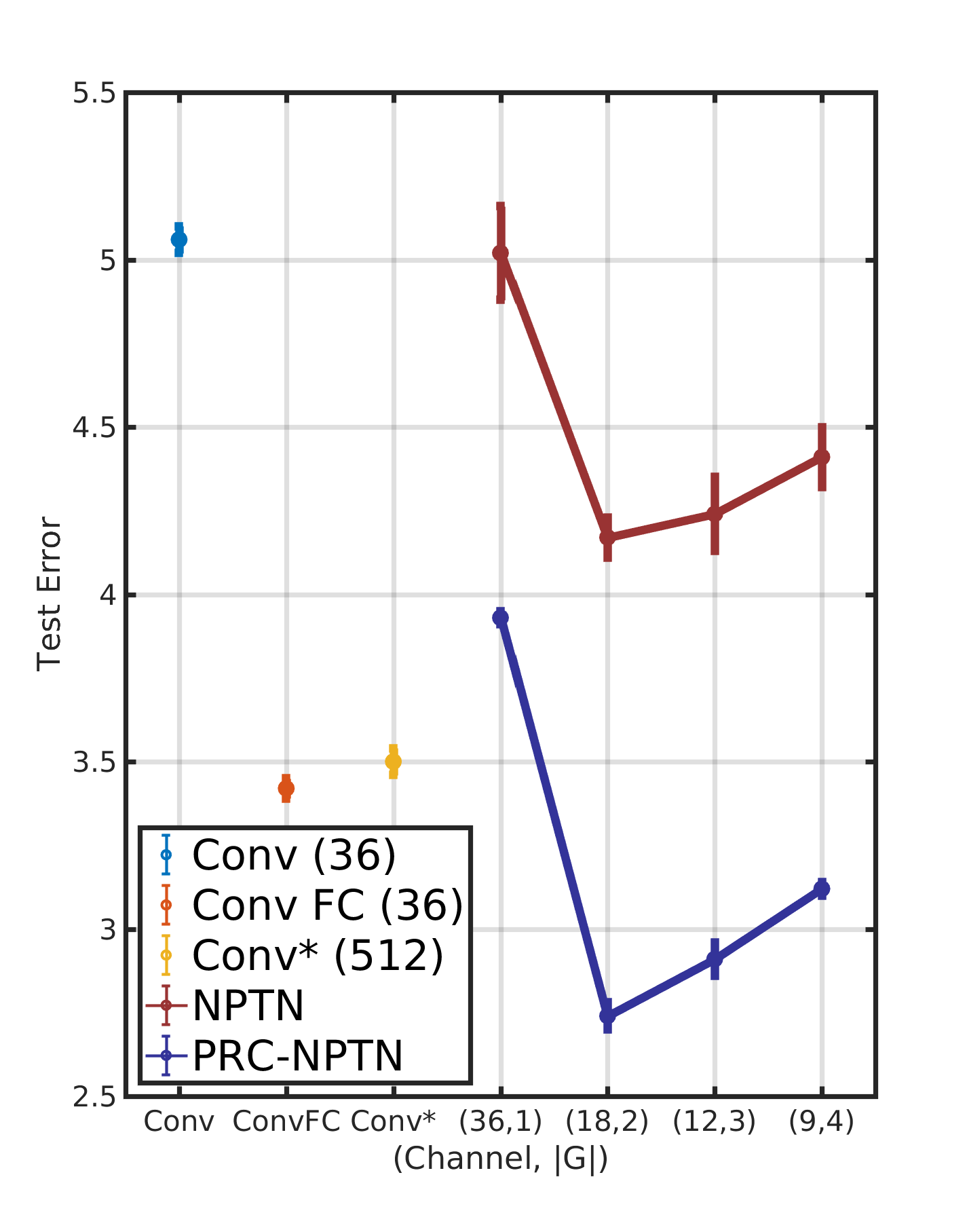}
        }
      \subfigure[($90^\circ$, $12$ pix)]{%
        \centering
         \includegraphics[width=0.23\columnwidth,valign=m]{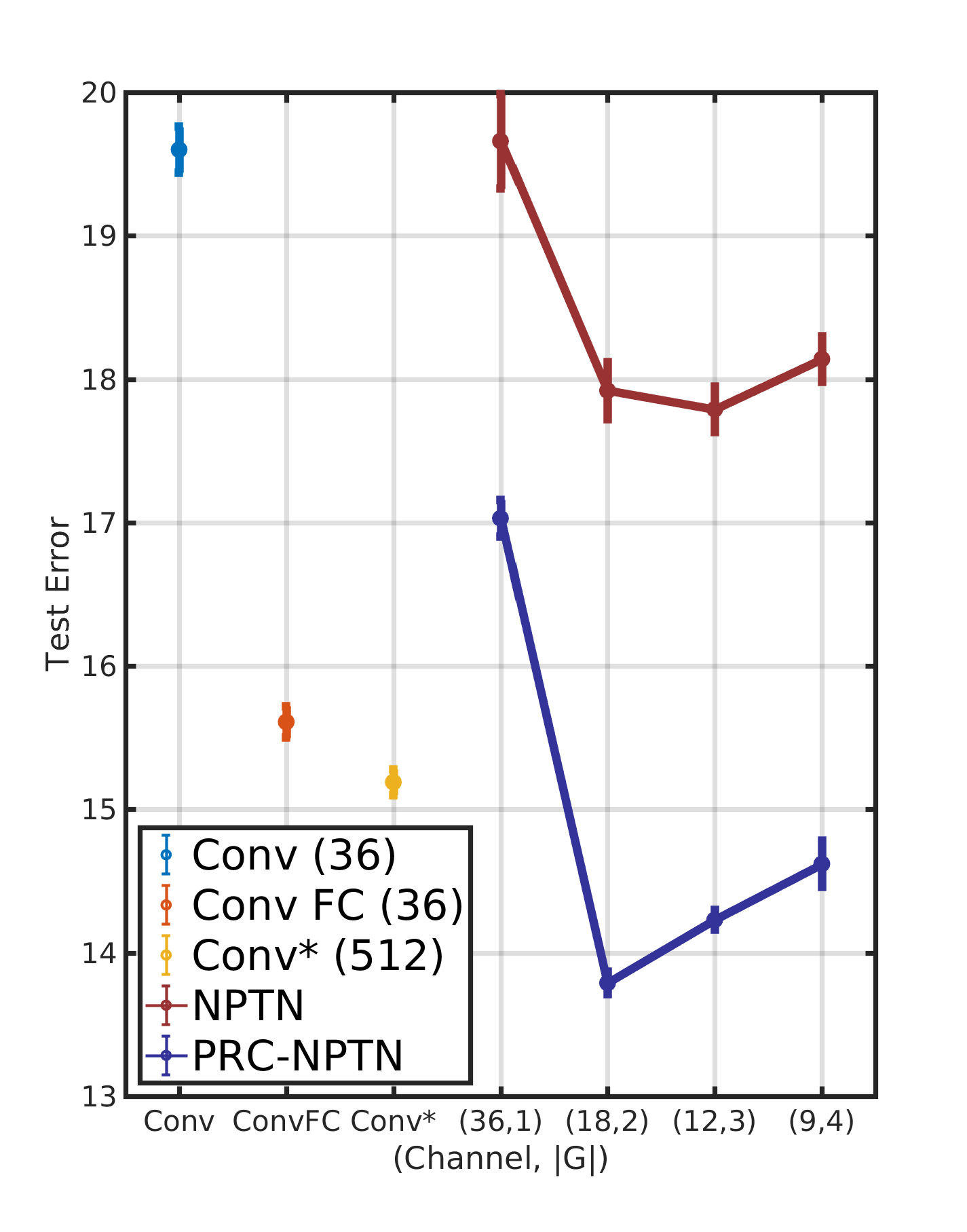}
        }      

    \end{center}
    \vspace{-0.5cm}
\caption{ \textbf{Simultaneous Transformation Results: } Test error statistics with mean and standard deviation on  MNIST with progressively extreme transformations with \textbf{random rotations} and \textbf{random pixel shifts simultaneously}.  For PRC-NPTN and NPTN the brackets indicate the number of channels in the layer 1 and $G$. Note that for this experiment, CMP=$|G|$. We provide experiments on CIFAR 10 with DenseNets in the supplementary.}
\label{fig_mnist_simul}

\end{figure}


\textbf{Efficacy in Learning Arbitrary and Unknown Transformations Invariances from Data.}  We evaluate on one of the most important tasks of any perception system, \emph{i.e.} being invariant to nuisance transformations \textit{learned} from the data itself. Most other architectures based on vanilla ConvNets learn these invariances through the implicit neural network functional map rather than explicitly through the architecture as PRC-NPTNs. Moreover, most previous approaches needed hand crafted architectures to handle different transformations. We benchmark our networks based on tasks where nuisance transformations such as large amounts of in-plane rotation and translation are steadily increased, with no change in architecture whatsoever. For this purpose, we utilize MNIST where it is straightforward to add such transformations without any artifacts. 

 We benchmark on such a task as described in \cite{pal2018nptn} and for fair comparisons, we follow the exact same protocol. We train \textit{and} test on MNIST augmented with progressively increasing transformations \emph{i.e.} \textbf{1)} extreme random translations (up to 12 pixels in a 28 by 28 image), \textbf{2)} extreme random rotations (up to $90^\circ$ rotations) and finally \textbf{3)} both transformations simultaneously. \textit{Both} train and test data were augmented randomly for every sample leading to an increase in overall complexity of the problem. No architecture was altered in anyway between the two transformations \emph{i.e.} they were not designed to specifically handle either. The same architecture for all networks is expected to learn invariances directly from data unlike prior art where such invariances are hand crafted in \cite{teney2016learning,li2017deep,sifre2013rotation,xu2014scale,cohen2016group,henriques2016warped}. 

For this experiment, we utilize a two layered network with the intermediate layer 1 having up to 36 channels and layer 2 having exactly 16 channels for all networks (similar to the architectures in \cite{pal2018nptn}) except a wider ConvNet baseline with 512 channels. All ConvNet, NPTN and PRC-NPTN models have the similar number of parameters (except the ConvNet with 512 channels). For PRC-NPTN, the number of channels in layer 1 was decreased from 36, through to 9 while $|G|$ was increased in order to maintain similar number of parameters. All PRC-NPTN networks have a two layered $1\times 1$ pooling network with same number of channels as that layer. For a fair benchmark, Convnet FC has 2 two-layered pooling networks with 36 channels each. Average test errors are reported over 5 runs for all networks.


\begin{table*}
\small
\centering
\begin{tabular}{l c } 
\hline
\hline 
Architecture   &   \\
\hline
\hline
ConvNet B    &  C(12) - C(24) - C(48) - C(48)  - GAP - FC(300) - FC(200) - FC(8)\\
NPTN-large B &  C(12) - NPTN(24) - NPTN(48) - NPTN(48) - GAP - FC(300) - FC(200) - FC(8)    \\ 
NPTN-small B &  C(12)  - NPTN(8)  - NPTN(24)  - NPTN(48) - GAP - FC(300) - FC(200) - FC(8)    \\
PRC-NPTN B &  C(12)  - PRC(24) - PRC(48)  - PRC(48) - GAP - FC(300) - FC(200) - FC(8) \\
\hline
ConvNet C  &   C(12) - C(24)  - C(48)  - C(64) - C(128)  - GAP - FC(8) \\

PRC-NPTN C  &   C(12) - PRC(24) - PRC(48) - PRC(64) - PRC(128) - GAP - FC(8) \\

\hline
\end{tabular}
\caption{\textbf{Architectures tested on ETH-80.} C - convolution layer, FC - fully connected layer, PRC - PRC-NPTN layer, NPTN - NPTN layer, GAP - global average pooling layer. Every Conv, NPTN and PRC-NPTN layer was followed by a spatial pooling layer of kernel size 2 except the last laer before the GAP.  The $1\times 1$ versions of these architectures have a $1\times$ conv layer after every $3\times 3$ layer except the first C(12) layer. ConvNet B was designed to be similar to the architecture explored in \cite{khasanova2017graph} however with more layers. ConvNet C was designed to be more aligned with modern architecture choices such as global average pooling followed by just one FC layer. }
\label{tab_exp_eth_arch}
\end{table*}

\textbf{Discussion.} We present all test errors for this experiment in Fig.~\ref{fig_minist_indi} and Fig.~\ref{fig_mnist_simul}\footnote{We display only the (12, 3) configuration for NPTN as it performed the best. The translation results and more benchmarks with NPTNs are provided in the supplementary. We obtain similar perforamnce improvements with extreme translation as well.}. From both figures, it is clear that as more nuisance transformations act on the data, PRC-NPTN networks outperform other baselines with the same number of parameters. In fact, even with significantly more parameters, ConvNet-512 performs worse than PRCN-NPTN on this task for all settings. Since the testing data has nuisance transformations similar to the training data, the only way for a model to perform well is to learn invariance to these transformations.  It is also interesting to observe that permanent \textit{random} connectomes do indeed help with generalization. Indeed, without randomization the performance of PRCN-NPTNs drop substantially. The performance improvement of PRC-NPTN also increases with nuisance transformations, showcasing the benefits arising from modelling such invariances. This is particularly apparent from Fig.~\ref{fig_mnist_simul}, where the two simultaneous nuisance transformations pose a significant challenge. Yet, as the transformations increase, the performance improvements increase as well. 


\begin{table*}
\centering
\begin{tabular}{l c c c c c} 
\hline
\hline 
Method (Protocol 1)   &   Accuracy (\%)   &   \#params  &  Factor  &  \#filters  &  Reduction   \\
\hline
\hline
ConvNet \cite{khasanova2017graph}  &    $93.69$     &    1.4M     &    &  230   &  - \\

NPTN* \cite{pal2018nptn}     &  $ 96.2$    &  1.4M       &      &    230     &  -  \\

\hline
ConvNet B      &   $95.61$    &  110K       & $1\times$     &    3780    & $1\times$  \\
ConvNet B ($1\times 1$)      &   $94.54$    &  115K       &  $0.95\times$   &   3780   & $1\times$  \\ 
NPTN-large B (3)      &   $94.63$    &  189K       &   $0.58\times$  &   11268     & $0.33\times$  \\
NPTN-small B (3)      &   $95.09$    &  120K       &   $0.91\times$  &   4356     & $0.86\times$  \\
PRC-NPTN B  (8, 2)      &   $\mathbf{96.72}$    &  \textbf{97K}      &   $\mathbf{1.13\times}$   &    \textbf{708}     & $\mathbf{5.33 \times}$  \\

\hline
ConvNet C      &   $93.90$    &  116K       &   $1\times$  &   12740     & $1\times$  \\
ConvNet C ($1\times 1$)      &   $95.98$    &  138K        &  $0.84\times$   &   12740     &  $1\times$  \\

PRC-NPTN C (8, 2)      &   $95.93$    &  64K       &   $1.81\times$  &  \textbf{1220}   &  $\mathbf{10.44\times}$  \\  
PRC-NPTN C (8, 4)      &   $\mathbf{96.40}$    &  \textbf{39K}      &  $\mathbf{2.97\times}$    &   \textbf{1220}   &   $\mathbf{10.44\times}$   \\  
\hline
\end{tabular}
\caption{\textbf{Test accuracy on ETH-80 Protocol 1.} $*$ indicates the result was obtained from the corresponding paper, and is not on the split used for our experiments. For NPTN is number in the bracket denotes $|G|$, for PRC-NPTN the numbers denote $|G|$ and CMP respectively.}
\label{tab_exp_eth}
\end{table*}



\textbf{Evaluation on the ETH-80 dataset} The ETH-80 dataset was introduced in \cite{leibe2003analyzing} as a benchmark to test models against 3D pose variation of various objects. The dataset contains 80 objects belonging to 8 different classes. Each object has images from different viewpoints on a hemisphere for a total of 41 images per object. The images were resized to $50\times 50$ following \cite{khasanova2017graph}. This dataset is perfectly poised to test how efficiently a model can learn invariance to 3D viewpoint variation.

\textbf{Protocol:} For this experiment, we follow the evaluation protocol as described in \cite{khasanova2017graph} \footnote{We also present results on a harder protocol we devised in the supplementary}. We randomly select 2,300 images to train and test on the rest. For a fair comparison we retrain the ConvNet described in \cite{khasanova2017graph}. We design two ConvNet architectures which reflect more modern architecture choices such as a smaller FC layer or having only a global average pooling after a number of conv layers. Table.~\ref{tab_exp_eth_arch} presents the architectures that we train for this experiment. Every conv layer (except the first conv layer \textit{within} a PRC-NPTN layer) is followed by BatchNorm and ReLU. We train corresponding PRC-NPTN models that have fewer parameters. For these experiments with PRC-NPTN, we replace the average pooling across channels (not max pooling or CMP) with a $1\times 1$ convolution layer. We do this to explore the effect of weighted pooling instead of vanilla channel average pooling. To maintain a fair comparison, we compare against equivalent ConvNet baselines with an extra $1\times 1$ added. We also perform an ablation study with the randomization removed. All models were trained with Adam with a learning rate of 0.01 for 100 epochs and a batch size of 64. Each architecture was trained 10 separate times with the mean of the runs being reported. We showcase the results in Table.~\ref{tab_exp_eth}.

\textbf{Discussion:} We find that of the two different types of architectures that we explore, PRC-NPTNs outperform both corresponding ConvNet architectures. Further, they do so not only with fewer number of parameters, but also fewer number of $3\times 3$ filters. In fact, PRC-NPTN C for $|G|=8$ and CMP=4 outperforms the corresponding ConvNet C architectures with a $2.97\times$ reduction in the number of parameters and a $10.44\times$ reduction in the number of $3\times 3$ convolution filters. Similarly, PRC-NPTNs outperform two architectures of NPTNs \cite{pal2018nptn} both with significantly fewer parameters and $3\times 3$ filters. These results illustrate that  PRC-NPTN can utilize filters and parameters more efficiently on a classification problem which requires 3D pose invariance. This efficiency we conjecture, is due to the fact that permanent random pooling results in an inductive bias that explicitly helps the learning of multiple invariances within the same layer thereby vastly increasing model capacity. Note that the almost 3X reduction in the number of parameters and 10X reduction in the number of filters is achieved without the use of any network pruning or post processing methods. Note that competing methods presented in \cite{khasanova2017graph} all perform comparably however with 1.4 million parameters each with the highest result being TIGradNet \cite{khasanova2017graph} at 95.1, HarmNet at 94.0 \cite{worrall2017harmonic}. PRC-NPTN outperforms these methods which were designed to invoke invariances through inductive biases while using a fraction of the number of parameters.

\newpage


\section{Appendix}
\textbf{Abstract:} In this supplementary material, we provide a proof for Lemma 2.1 in the main paper, a more complete table of results including different parameter comparisons of NTPN baselines, timing results of our more efficient CUDA implementation and the pseudocode for PRCN-NPTNs. Further, we present results on the dataset ETH-80 consisting of 3D viewpoint variations of objects in comparison with previous work. Finally, we present results on CIFAR 10 with vanilla DenseNets and PRCN applied to DenseNets to form DensePRC-NPTNs.

\textbf{Prior Art using Alternate Architectures.}  Several works have explored alternate deep layer architectures. A few of the main developments were the application of the skip connection \cite{he2016deep}, depthwise separable convolutions \cite{chollet2017xception} and group convolutions \cite{xie2017aggregated}. Randomly initialized channel shuffling is an operation that is central to the application of permanent random connectomes. However, \textit{deterministic} non-randomized channel shuffling was explored while optimizing networks for computation efficiency \cite{zhang2018shufflenet}. Nonetheless, none of these methods explored \textit{permanent} and \textit{ random } connectomes from the perspective of explicitly learning invariances from data itself. 

\textbf{Invariances in a PRC-NPTN layer.}  Recent work introducing NPTNs \cite{pal2018nptn} had highlighted the Transformation Network (TN) framework in which invariance is generated during the forward pass by pooling over dot-products with transformed filter outputs. A vanilla convolution layer with a single input and output channel (therefore a single convolution filter) followed by a $k\times k$ \textit{spatial} pooling layer can be seen as a single TN node enforcing translation invariance with the number of filter outputs being pooled over to be $k\times k$. It has been shown that $k\times k$ spatial pooling over the convolution output of a single filter is an approximation to channel pooling across the outputs of $k\times k$ translated filters \cite{pal2018nptn}. The output $\Upsilon(x)$ of such an operation with an input patch $x$ can be expressed as 
\begin{align}
    \Upsilon(x) =  \max_{g\in \mathcal{G}}\langle  x, gw  \rangle
\end{align}
where $\mathcal{G}$ is the set of filters whose outputs are being pooled over. Thus, $\mathcal{G}$ defines the set of transformations and thus the invariance that the TN node enforces. In a vanilla convolution layer, this is the translation group (enforced by the convolution operation followed by \textit{spatial} pooling).  An NPTN removes any constraints on $\mathcal{G}$ allowing it to approximately model arbitrarily complex transformations. A vanilla convolution layer would have \textit{one} filter whose convolution is pooled over spatially (for translation invariance). In contrast, an NPTN node has $|\mathcal{G}|$ \textit{independent} filters whose convolution outputs are pooled across \textit{channel} wise leading to general invariance.

A PRC-NPTN layer inherits the property from NPTNs to learn arbitrary transformations and thereby arbitrary invariances using $\mathcal{G}$. Individual channel max pooling (CMP) nodes act as NPTN nodes sharing a \textit{common} filter bank as opposed to independent and disjoint filter banks for vanilla NPTNs. This allows for greater activation sharing, where transformations learned from data through one subset of filters can be used for invoking similar invariances in a parallel computation path. This sharing and reuse of activation maps allows for higher parameter and sample efficiency. As we find in our experiments, randomization plays a critical role here, allowing for a simple and quick approximation to obtaining high performing invariances. A high activation map can activate multiple CMP nodes, winning over multiple sub-sets of low activations. Gradients flow back to these winning activations updating the filters to further model the features observed during that particular batch. Note that, CMP nodes in the same layer can pool over disjoint subsets to invoke a variety of invariances, leading to a more versatile network and also better modelling of a particular kind of invariance as we find in our experiments. Further, the primary source of invoking invariances in NPTN was understood to be the symmetry of the unitary group action space \cite{pal2018nptn}. General invariances were assumed to be only approximately forming a group. Lemma~\ref{lem_invariance} shows that group symmetry is not necessary to reduce variance of the quantity $\max \Upsilon(x)$ due to the action of the set elements $g$ on some test input patch $x$. Though, the result makes a strong assumption regarding the distribution of $\Upsilon(x)$, it to the best of our knowledge the first result of its kind to show increased invariance without a group symmetric action.

\begin{figure}
    \begin{center}
         \subfigure[Depth]{%
        \centering
        \includegraphics[width=0.2\columnwidth,valign=m]{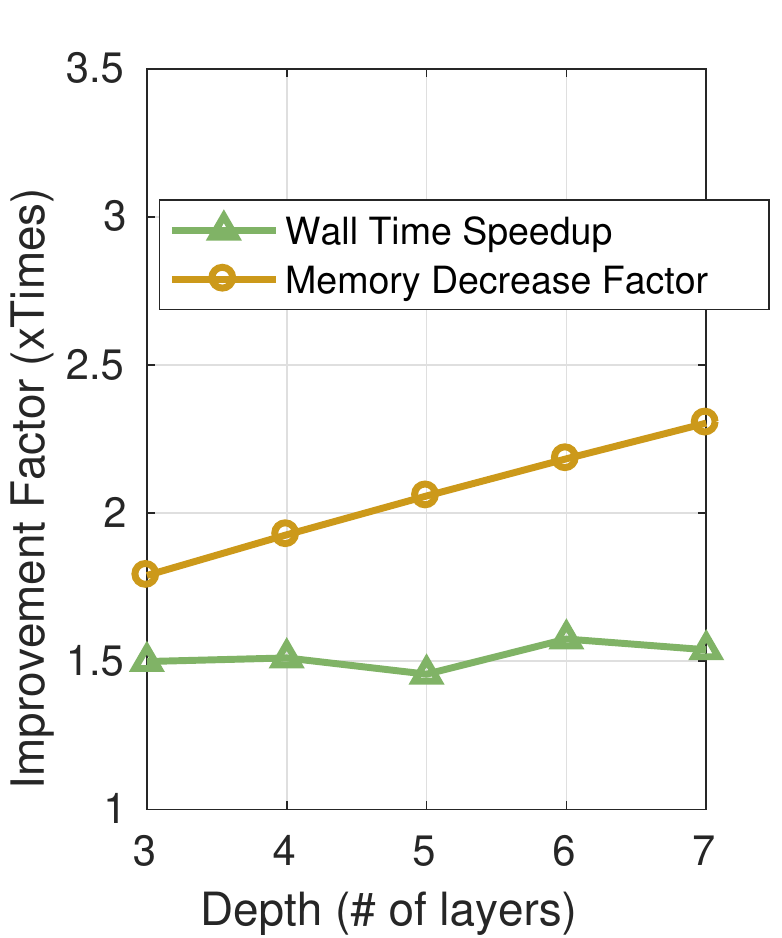}
        }
        \subfigure[CMP]{%
        \centering
            \includegraphics[width=0.2\columnwidth,valign=m]{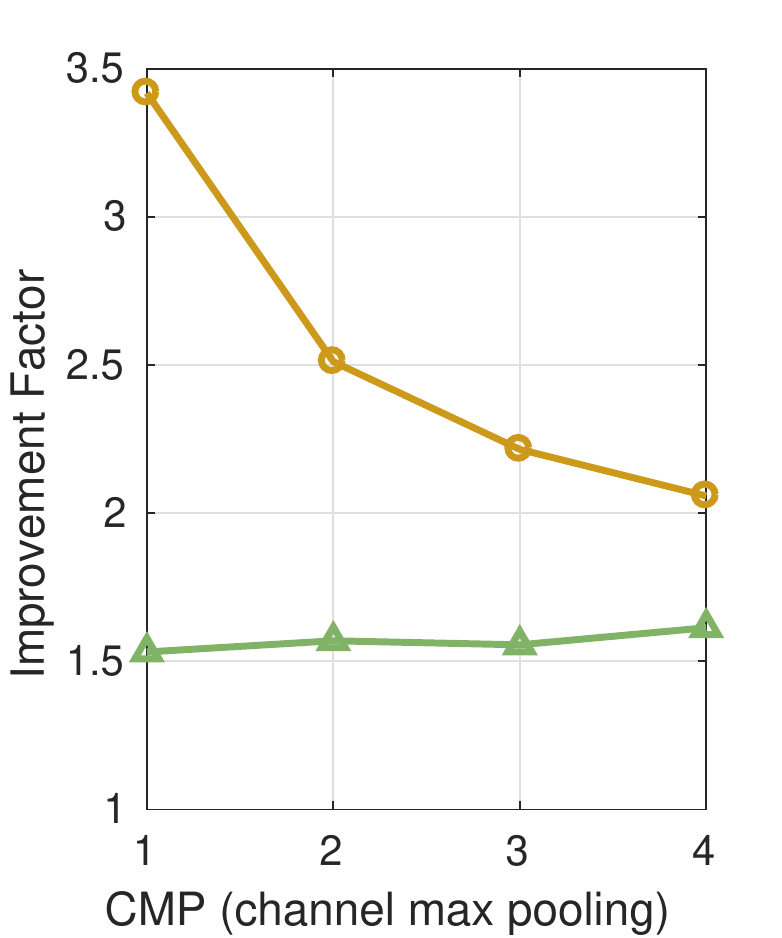}
        }
        \subfigure[Width]{%
        \centering
         \includegraphics[width=0.2\columnwidth,valign=m]{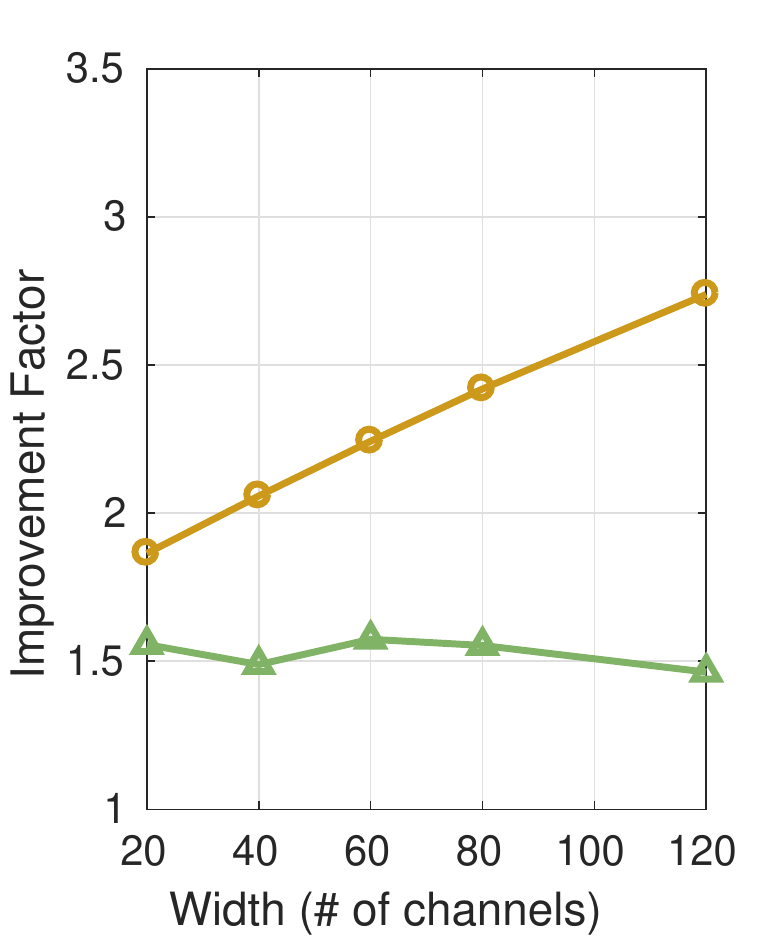}
        }
      \subfigure[Growth Factor]{%
        \centering
         \includegraphics[width=0.2\columnwidth,valign=m]{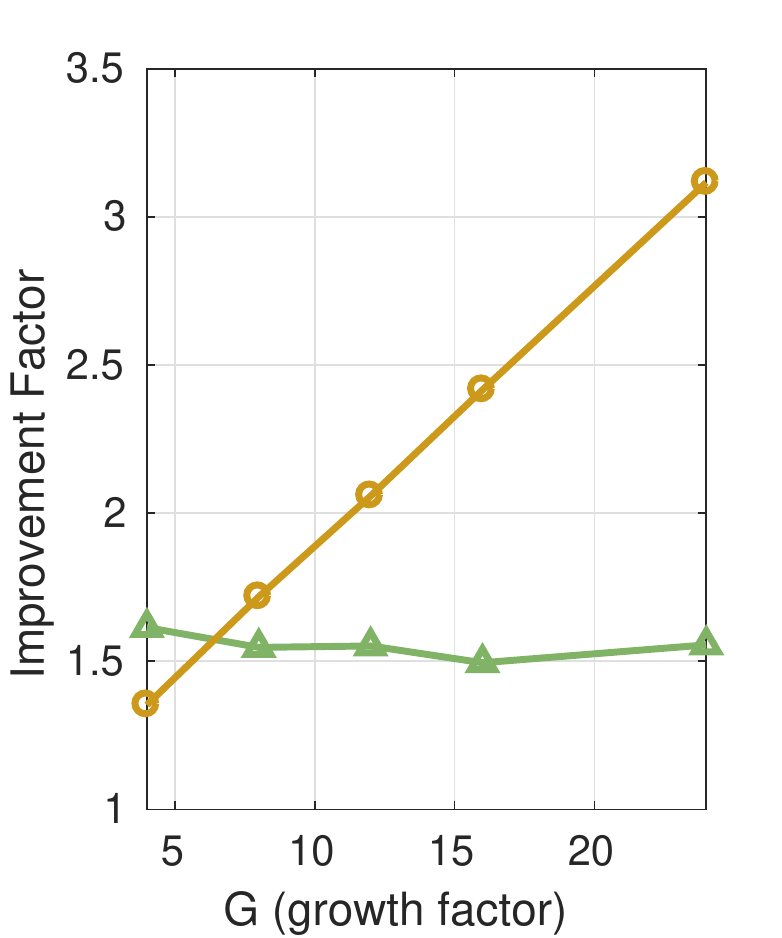}
        }      

    \end{center}
    \vspace{-0.5cm}
\caption{ Computational efficiency improvements of our CUDA kernel implementations.  }
\label{fig_timing}
    \vspace{-0.2cm}

\end{figure}

\textbf{Efficacy on CIFAR10 Image Classification.}  MNIST was a good candidate for the previous experiment where the addition of nuisance transformations such as translation and rotation did not introduce any artifacts. However, in order to validate permanent random connectomes on more realistic data, we utilize the CIFAR10 dataset and AutoAugmentation \cite{cubuk2018autoaugment} as the nuisance transformation. Note that, from the perspective of previous works in network invariance, it is unclear how to hand craft architectures to handle invariances due to the variety of transformations that AutoAugment invokes. Here is where the general invariance learning capability of PRC-NPTNs would help, without the need of expertise in such hand-crafting. 

We replace vanilla convolution layers with kernel size 3 in DenseNets with PRC-NPTNs without the 2-layered pooling networks. There was another modification for this experiment. For each input channel of a layer, a total of $|G|=12$ filters were learnt. However only a few of them were pooled over (channel max pool or CMP). We pool with CMP = 1, 2, 3 or 4 channels randomly keeping $|G|=12$ fixed always. Note that in contrast with the MNIST experiment, pooling was always done over $|G|$ number of channels (CMP=$|G|$). This provides a different setting under which PRC-NPTN can be utilized. All models in this experiment were trained with AutoAugment and were tested on both a) the original testing images and also on b) the test set transformed by AutoAugment. Similarly to the previous experiment, a model would have learn invariance towards these auto-augment transformations in order to perform well. All DenseNet models have 12 layers with the PRC-NPTN variant having the same number of parameters to enable us to perform multiple runs in a reasonable amount of time. The lower accuracy compared to other studies can be accounted by this. We train 5 models for each setting and report the mean and standard deviation of the errors. Training 5 runs for each of the hyperparameter combination to account for the randomization is yet another reason which tended to result in unreasonably large experiment times. Importantly, the goal of this experiment is not to push the state-of-the-art, but rather to investigate the behavior of DensePRC-NPTNs within the limits of computational resources available for this study while executing 5 runs for each network.  

\textbf{Discussion.}  Table.~\ref{tab_exp_cifar10} presents the results of this experiment. We find PRC-NPTN provides clear benefits even with architectures employing heavy use of skip connections such as DenseNets with the same number of parameters. Performance seems to increase as channel max pooling increased. Further, randomization seems to be important to the overall architecture even when given the complex nature of real image transformations. PRC-NPTN helps DenseNets account for nuisance transformations better even for those as extreme as auto-augment with its 16 transformation types  ShearX/Y, TranslateX/Y, Rotate, AutoContrast, Invert, Equalize, Solarize, Posterize, Contrast, Color, Brightness, Sharpness, Cutout, Sample Pairing to various degrees.  With these evidence, it is interesting to find that random connectomes can be motivated from the perspective of learning heterogeneous invariances from data without any change in architectures. We find that they provide a promising alternate dimension in future network design in contrast to the ubiquitous use of highly structured and ordered connectomes.

\begin{figure}
    \begin{center}
         \subfigure[Trans $0$ pix]{%
        \centering
        \includegraphics[width=0.23\columnwidth,valign=m]{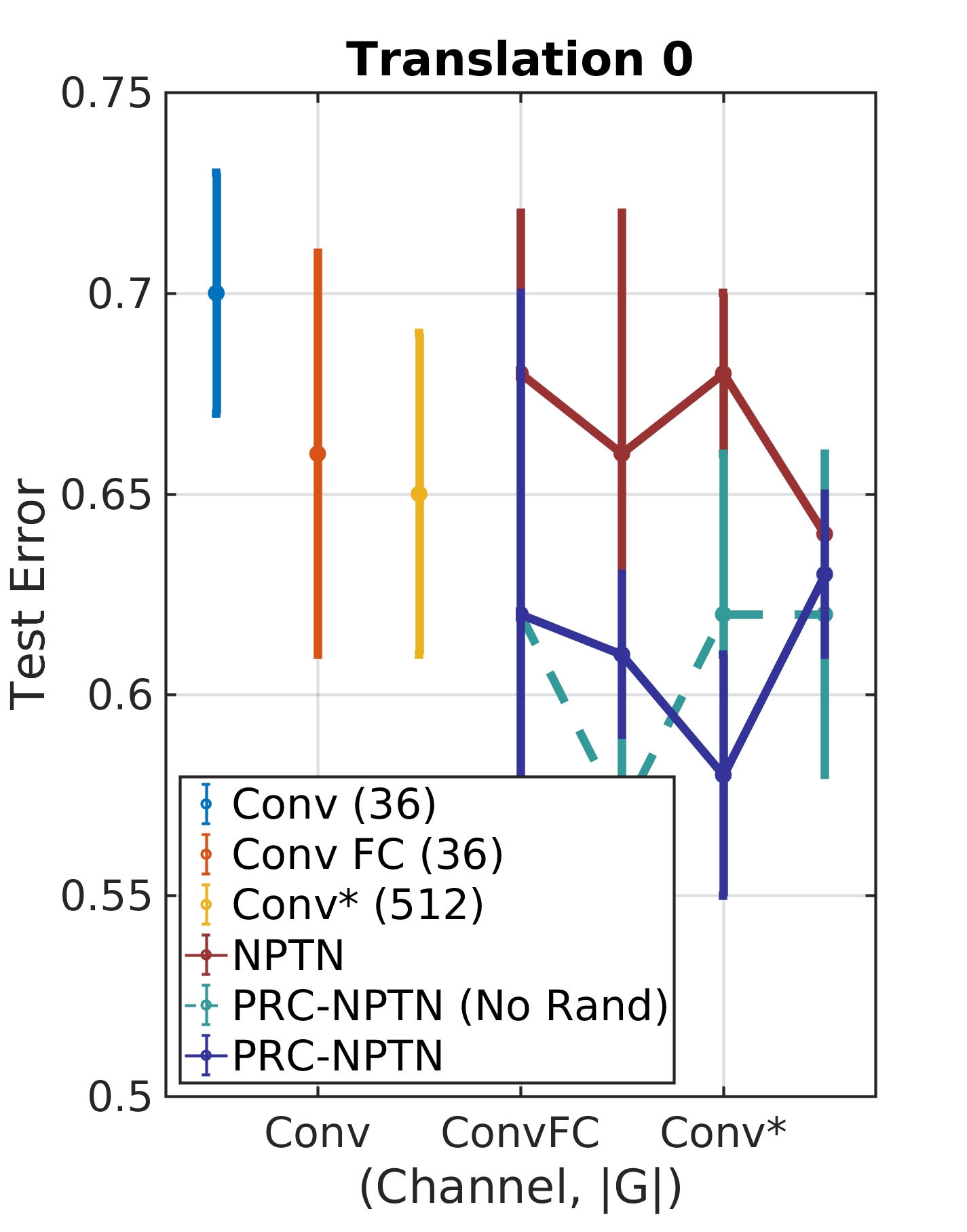}
        }
        \subfigure[Trans $4$ pix]{%
        \centering
            \includegraphics[width=0.23\columnwidth,valign=m]{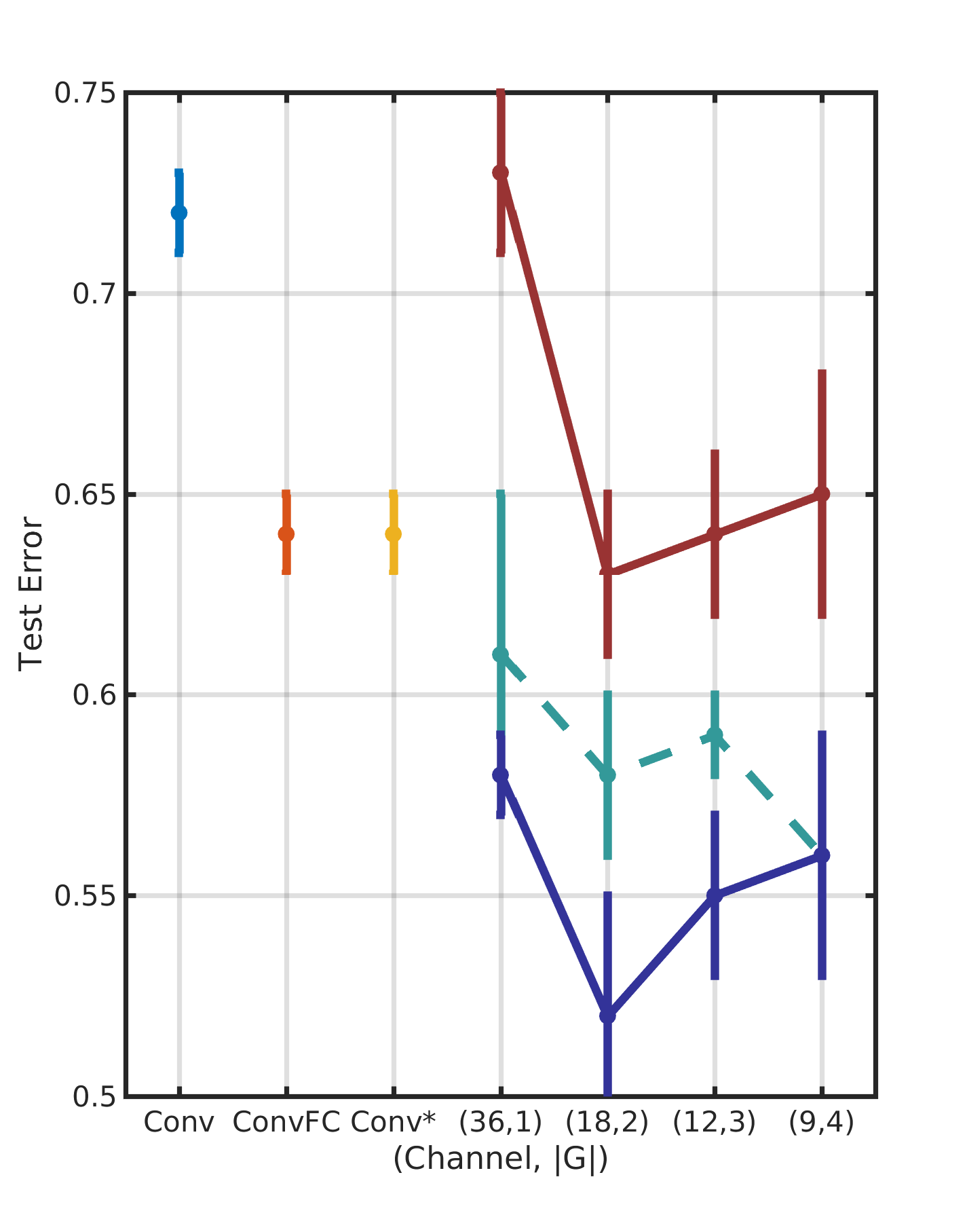}
        }
        \subfigure[Trans $8$ pix]{%
        \centering
         \includegraphics[width=0.23\columnwidth,valign=m]{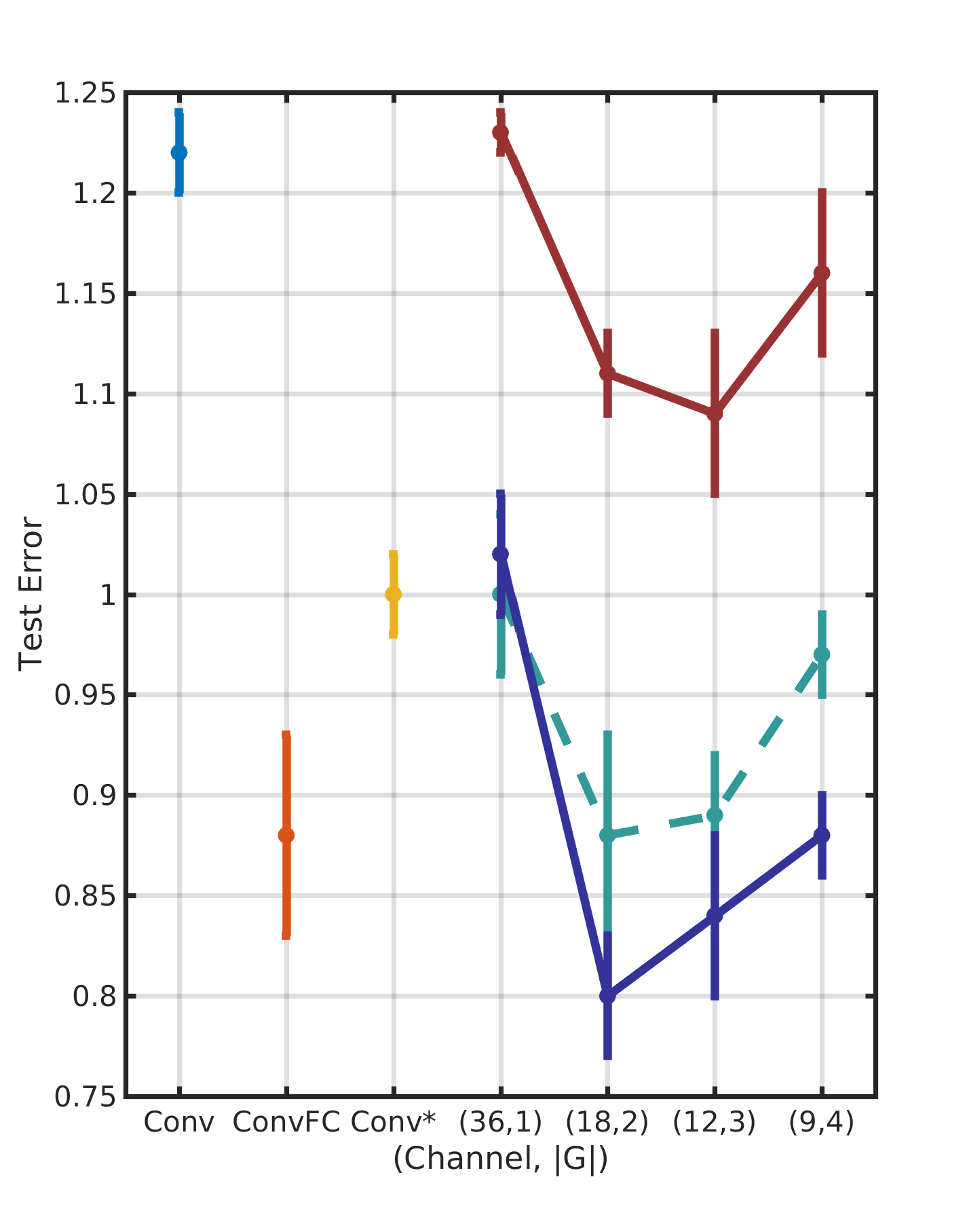}
        }
      \subfigure[Trans $12$ pix]{%
        \centering
         \includegraphics[width=0.23\columnwidth,valign=m]{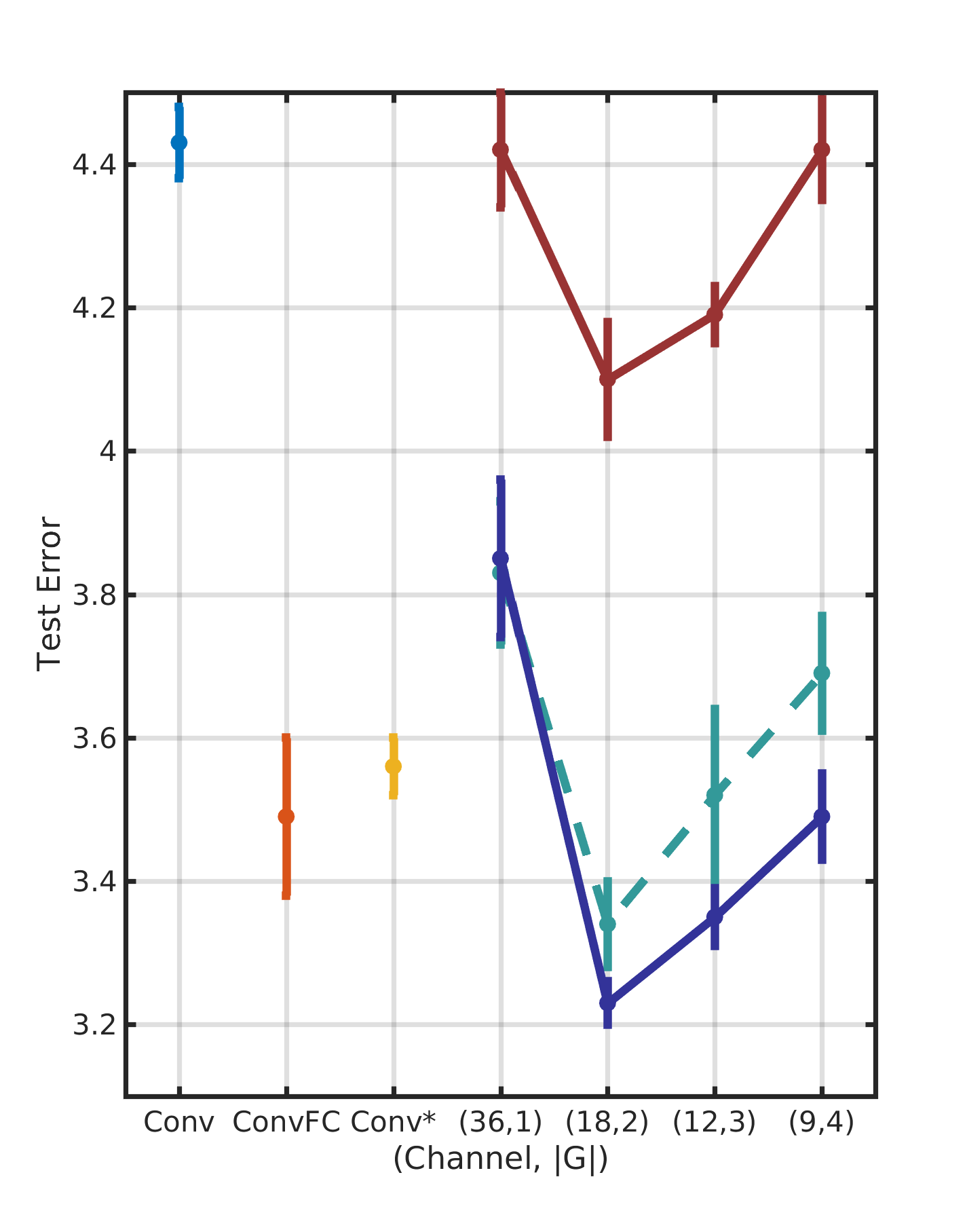}
        }      

    \end{center}
    \vspace{-0.5cm}
\caption{ \textbf{Individual Transformation Results: }Test error statistics with mean and standard deviation on  MNIST with progressively extreme transformations with \textbf{random pixel shifts}. For PRC-NPTN and NPTN the brackets indicate the number of channels in the layer 1 and $G$. ConvNet FC denotes the addition of a 2-layered pooling $1\times 1$ pooling network after every layer. Note that for this experiment, CMP=$|G|$. Permanent Random Connectomes help with achieving better generalization despite increased nuisance transformations. }
\label{fig_minist_indi}

\end{figure}

\begin{table*}
\tiny
\centering
\begin{tabular}{l | c | c  | c | c | c | c  } 
\hline\hline 
Method   &   CIFAR10  &  (w/o Random)  &    CIFAR 10++   &    (w/o Random)   &   Speed  &  Memory     \\
\hline
\hline



DenseNet-Conv  &    $11.47_{\pm 0.19} $  &  -       &    $21.37_{\pm 0.29} $    &   -   &       &       \\
\hline
DensePRC-NPTN (CMP=1)    &    $ 11.82 _{\pm 0.20 }$   &    $  13.33  _{\pm 0.23  } $      &$22.03  _{\pm   0.08 }$     &      $ 23.88   _{\pm 0.38  } $   &    1.29x   &     2.92x   \\

DensePRC-NPTN (CMP=2)     &   $10.78 _{\pm 0.31}$     &      $  11.67  _{\pm 0.36 } $     &    $20.71 _{\pm 0.23}$   & $ 21.90 _{\pm 0.33  } $     &    1.34x   &    1.96x    \\
DensePRC-NPTN (CMP=3)     &   $10.95 _{\pm 0.12}$     &      $  11.59  _{\pm 0.23  } $     &    $20.95 _{\pm 0.20}$    &       $  21.80  _{\pm 0.42  } $     &  1.36x     &     1.64x   \\
DensePRC-NPTN (CMP=4)    &    $\mathbf{10.61 _{\pm 0.11}}$   &    $  11.41  _{\pm 0.12  } $      &        $\mathbf{20.80 _{\pm 0.12}}$     &      $  21.47  _{\pm  0.16 } $    & 1.36x      &   1.48x     \\

\hline
\end{tabular}
\caption{\textbf{Efficacy on CIFAR10:} Test error statistics on CIFAR10 with mean and standard deviation. ++ indicates AutoAugment testing. Each DenseNet and its corresponding PRC-NPTN variant has the same number of parameters. $|G|=12$ for PRC-NPTN and growth rate was kept at 12 for DenseNet-Conv. (w/o Random) indicates no randomization in the connectomes constructed (as an ablation study). The speed and memory improvements are multiplicative improvement factors of our CUDA kernel implementation compared to baseline optimized PyTorch code. }
\label{tab_exp_cifar10}
\vspace{-0.6cm}
\end{table*}


\begin{figure}
    \begin{center}
            \includegraphics[width=0.95\columnwidth,valign=m]{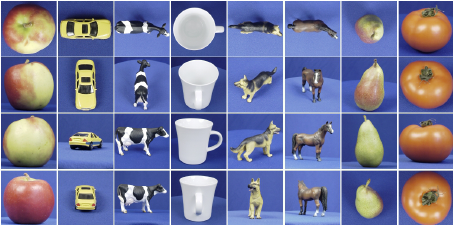}
    \end{center}
    \vspace{-0.2cm}
\caption{ Sample images from the ETH-80 database. The dataset contains 80 objects belonging to 8 different classes. Each object has images from different viewpoints on a hemisphere resulting in 3D pose and viewpoint variation for each object.}
\label{fig_eth_80}
 \vspace{-0.5cm}

\end{figure}

\begin{table*}
\centering
\small
\begin{tabular}{l c c c c c} 
\hline
\hline 
Method (Protocol 2)    &   Accuracy (\%)   &   \#parameters  &  Reduction  &  \#filters ($3\times 3$)  &  Reduction   \\
\hline
\hline
ConvNet \cite{khasanova2017graph}  &    $84.49$     &    1.4M     &    &  230   &  - \\


\hline
ConvNet B      &   $95.58$    &  110K       & $1\times$     &    3780    & $1\times$  \\
ConvNet B ($1\times 1$)      &   $94.33$    &  115K       &  $0.95\times$   &   3780   & $1\times$  \\ 
NPTN-large B ($3$)      &   $94.58$    &  189K       &   $0.58\times$  &   11268     & $0.33\times$  \\
NPTN-small B ($3$)      &   $94.70$    &  120K       &   $0.91\times$  &   4356     & $0.86\times$  \\
PRC-NPTN B  ($8$, 2)      &   $94.34$    &  \textbf{97K}      &   $\mathbf{1.13\times}$   &    \textbf{708}     & $\mathbf{5.33 \times}$  \\

\hline
ConvNet C      &   $95.44$    &  116K       &   $1\times$  &   12740     & $1\times$  \\
ConvNet C ($1\times 1$)      &   $94.40$    &  138K        &  $0.84\times$   &   12740     &  $1\times$  \\

PRC-NPTN C ($8$, 2)      &   $94.68$    &  64K       &   $1.81\times$  &  \textbf{1220}   &  $\mathbf{10.44\times}$  \\  
PRC-NPTN C ($8$, 4)      &   $\mathbf{95.63}$    &  \textbf{39K}      &  $\mathbf{2.97\times}$    &   \textbf{1220}   &   $\mathbf{10.44\times}$   \\  
\hline
\end{tabular}
\caption{\textbf{Test accuracy on ETH-80 Protocol 2.}  For NPTN is number in the bracket denotes $|G|$, for PRC-NPTN the numbers denote $|G|$ and CMP respectively. }
\label{tab_exp_eth_prot2}
\end{table*}


\textbf{Evaluation on the ETH-80 dataset} The ETH-80 dataset was introduced in \cite{leibe2003analyzing} as a benchmark to test models against 3D pose variation of various objects. The dataset contains 80 objects belonging to 8 different classes. Each object has images from different viewpoints on a hemisphere for a total of 41 images per object. The images were resized to $50\times 50$ following \cite{khasanova2017graph}. This dataset is perfectly poised to test how efficiently a model can learn invariance to 3D viewpoint variation.

\textbf{Protocol:} For this experiment, we  devise a new protocol in which we train on one half of the horizontal views and test on the other half. We randomly split the vertical views between these train and test sets. We test with the same set of architectures as in the main paper with the same experimental settings. This protocol is harder since the other side of the object is not seen along with the objects not being symmetrical. Results are shown in Tab.~\ref{tab_exp_eth_prot2}. The results follow a similar tread to that of the original protocol described in the main paper. We find PRC-NPTN outperforms all other methods using a almost 3X less parameters and almost 10X fewer 3x3 filters. These results show that PRC-NPTN are more effective in learning even out of the plane invariances from data itself without any change in architecure.

\subsection{Proof of Lemma 2.1}

\begin{lemma}\label{lem_invariance} (Invariance Property) Assume a  novel test input $x$ and a  filter $ w$ both fixed vectors $ \in \mathbb{R}^d$. Further, let $g$ denote a random variable representing unitary operators with some distribution. Finally, let  $\Upsilon(x) =  \langle  x, gw  \rangle$, with $\Upsilon(x)   \sim  U(a,b) $   \emph{i.e.} a Uniform distribution between $a$ and $b$. Then, we have $$ \text{Var}(\max \Upsilon(x)) \leq  \text{Var}( \Upsilon (x) ) = \mathrm{Var} (\langle  g^{-1}x, w  \rangle) $$

\end{lemma}

\begin{proof}

Let $X$ be the random variable representing the randomness in $\langle x, gw    \rangle$ for fixed $x, w$ and random $g$. We assume that $X\sim U(0,1)$.

Considering a sample set $X_1, X_2 ... X_n \sim U(0,1)$ , then $X_{(n)} = \max_{1 \leq i \leq n X_n}$.
Now,
\begin{align}
    P(X_{(n)}  \leq x ) &= P(X_i \leq x, \forall i)\\
    &= P(X_i \leq x)^n \\
    &= x^n
\end{align}

Let the density of $X_{(n)}$ be denoted by $f_{X_{(n)}}(x)$, then

$f_{X_{(n)}}(x) = \left\{   \begin{array}{cc}
   0  & x \leq 0 \\
nx^{n-1}     & 0\leq x \leq 1\\
1 & x \geq 1
\end{array}  \right.$

Now, 

$$\mathbb{E}[X_{(n)}] =  \int_0^1 xnx^{n-1} dx = \frac{x^{n+1}}{n+1}n |_0^1 = \frac{n}{n+1}$$
$$\mathbb{E}[X_{(n)}^2] =  \int_0^1 x^2nx^{n-1} dx = \frac{x^{n+2}}{n+2}n |_0^1 = \frac{n}{n+2} $$
Therefore, 
$$ \mathrm{Var}(X_{(n)}) = \frac{n}{n+1} - \left(\frac{n}{n+1}\right)^2 = \frac{n}{(n+2)(n+1)^2}$$
Since the variance of $U(0,1)$ is $\frac{1}{12}$ \emph{i.e.} $\mathrm{Var(X_i)} = \frac{1}{12}$, and  $\mathrm{Var}(X_{(n)})$ is a decreasing function in $n$, along with the fact that $\mathrm{Var}(X_{(n)})$ for $n=1$ is $\frac{1}{12}$, we have 

$$\mathrm{Var}(X_{(n)})  \leq \mathrm{Var}(X_i)$$

For general $U(a,b)$, it follows shortly after considering $Y_i = \frac{X_i - a}{b-a}$ and that $\mathrm{Var}(Y_i)  = \mathrm{Var}(X_i)$. Finally, due to unitary $g$, $\langle x, gw  \rangle = \langle g^{-1}x, w  \rangle$.
\end{proof}

  \begin{figure}[t]
\footnotesize
    \begin{center}
\begin{algorithmic}[1]\label{alg_1}
    
\STATE\textbf{class PRCN-NPTN:}\\
\STATE	\textbf{def init(self, inch, outch, G, CMP, kernelSize, padding, stride):}\\
\STATE		~~~~self.G = G\\
\STATE		~~~~self.maxpoolSize = CMP\\
\STATE ~~~~self.avgpoolSize = int((inch*self.G)/(self.maxpoolSize*outch))\\
\STATE		~~~~self.expansion = self.G*inch\\
\STATE		~~~~self.conv1 = nn.Conv2d(inch, self.G*inch, kernelSize=kernelSize, groups=inch, padding=padding, bias=False)\\
\STATE		~~~~self.transpool1 = nn.MaxPool3d((self.maxpoolSize, 1, 1))\\
\STATE		~~~~self.transpool2 = nn.AvgPool3d((self.avgpoolSize, 1, 1))\\
\STATE		~~~~self.index = torch.LongTensor(self.expansion).cuda()\\
\STATE		~~~~self.randomlist = list(range(self.expansion))\\
\STATE		~~~~random.shuffle(self.randomlist)\\
\STATE		~~~~for ii in range(self.expansion):\\
\STATE		~~~~		~~~~self.index[ii] = self.randomlist[ii]\\
\STATE
\STATE	\textbf{def forward(self, x):}\\
\STATE		~~~~out = self.conv1(x)  \#inch $\longrightarrow$ G*inch\\
\STATE		~~~~out = out[:,self.index,:,:] \# randomization\\
\STATE		~~~~out = self.transpool1(out) \# G*inch $\longrightarrow$ inch*G/maxpool\\
\STATE		~~~~out = self.transpool2(out) \# inch*G/(maxpool*meanpool) $\longrightarrow$ outch\\
\STATE		~~~~return out

\end{algorithmic}
    \end{center}
    \caption{PRC-NPTN pseudo-code.}
    \label{alg:CoSaMP}
\end{figure}

\begin{table*}[t]
\tiny
\centering
\begin{tabular}{l | c c | c c | c c | c c } 
\hline\hline 
Rotation   &     $0^\circ$  &   *** &   $30^\circ$  &     ***  &   $60^\circ$   &    ***& $90^\circ$  &     ***\\
\hline
\hline

ConvNet (36) &   $  0.70 _{\pm  0.03  }$   &    -    &  $  0.92 _{\pm  0.03  }$   &   -    &   $ 1.32  _{\pm  0.07  }$     &     -   &     $ 1.93  _{\pm 0.02    }$        &  -   \\
ConvNet (36) FC &   $  0.66 _{\pm  0.05  }$   &  -      &   $  0.80 _{\pm  0.03  }$   &  -      &  $ 1.08  _{\pm  0.02  }$     &     -    &    $ 1.58  _{\pm 0.01    }$          &  - \\
ConvNet (512) &   $  0.65 _{\pm  0.04  }$   &  -      &   $ 0.80  _{\pm 0.02   }$   &    -     & $  1.14 _{\pm  0.03  }$     &   -     &     $ 1.54  _{\pm 0.03    }$         &  -  \\

\hline

NPTN (36,1) &    $  0.68 _{\pm 0.04   }$     &  -     &   $ 0.93  _{\pm 0.01   }$     &  -     &  $  1.35 _{\pm  0.05  }$   &  -         &     $  1.92 _{\pm 0.02    }$    &  -        \\

NPTN (18,2) &    $ 0.66  _{\pm 0.02   }$   &  -       &   $  0.87 _{\pm  0.04  }$     &  -     &  $ 1.18  _{\pm  0.02  }$    &  -        &     $ 1.67 _{\pm  0.03  }$     &  -       \\
NPTN (12,3) &   $ 0.68  _{\pm  0.06  }$   &     -    &  $ 0.84  _{\pm 0.02    }$   &   -     &  $ 1.19  _{\pm 0.01   }$     &    -      &   $1.64   _{\pm 0.02}$          &  - \\
NPTN (9,4) &    $ 0.64  _{\pm  0.01  }$     &  -     &   $ 0.88  _{\pm 0.05    }$    &  -      &  $ 1.21  _{\pm 0.03   }$     &  -       &     $ 1.65  _{\pm 0.02    }$     &  -       \\

\hline

PRCN (36,1) &     $  0.62 _{\pm 0.08   }$   &      $ 0.62 _{\pm 0.06 } $       &  $ 0.84  _{\pm 0.01   }$   &       $ 0.83 _{\pm 0.03 } $      & $ 1.17  _{\pm 0.05   }$     &    $ 1.19 _{\pm 0.02 } $         &    $ 1.72  _{\pm 0.05   }$           &   $ 1.73 _{\pm 0.06 } $   \\  
PRCN (18,2) &    $  0.61 _{\pm 0.02   }$   &      $ 0.57 _{\pm 0.02 } $      &   $ \mathbf{ 0.68  _{\pm 0.02   }}$   &     $ 0.73 _{\pm 0.02 } $       &  $  \mathbf{0.93  _{\pm 0.04    }}$     &      $ 0.99 _{\pm 0.04 } $       &    $  \mathbf{1.24  _{\pm 0.01   }}$        &    $ 1.33 _{\pm 0.02 } $     \\
PRCN (12,3) &    $ \mathbf{ 0.58 _{\pm 0.03   }}$   &    $  0.62_{\pm 0.04 } $     &   $ 0.72  _{\pm 0.02   }$   &      $  0.74 _{\pm 0.02 } $      &  $ 0.95  _{\pm 0.01    }$     &      $ 1.04 _{\pm 0.04 } $       &    $  1.28 _{\pm 0.01   }$         &   $ 1.33 _{\pm 0.01 } $    \\
PRCN (9,4) &  $  0.63 _{\pm 0.02   }$   &        $ 0.62 _{\pm 0.04 } $      & $ 0.75  _{\pm 0.03   }$   &     $ 0.77 _{\pm 0.02 } $      &   $ 0.99  _{\pm 0.03    }$     &    $ 1.05 _{\pm 0.03 } $         &    $ 1.31  _{\pm 0.03   }$          &   $  1.40 _{\pm 0.03 } $    \\

\hline
\hline
Translations   &     0 pixels &    ***  &  4 pixels  &      ***    &   8 pixels   &       ***  & 12 pixels  &         *** \\
\hline

ConvNet (36) &   $  0.69 _{\pm  0.04  }$   &      -    & $  0.72 _{\pm  0.01  }$   &     -   &  $ 1.22  _{\pm  0.02  }$     &       -   &   $ 4.43  _{\pm 0.05    }$         &  -  \\
ConvNet (36) FC &   $  0.60 _{\pm  0.02  }$   &      -   &  $  0.64 _{\pm  0.01  }$   &  -     &   $ 0.88  _{\pm  0.05  }$     &    -     &    $ 3.49  _{\pm 0.11    }$         &    \\
ConvNet (512) &   $   0.63 _{\pm 0.02    }$   &       -  &  $ 0.64  _{\pm 0.01   }$   &  -      &  $ 1.00  _{\pm 0.02    }$     &    -     &    $ 3.56  _{\pm 0.04   }$         &   - \\
\hline
NPTN (36,1) &    $  0.68 _{\pm 0.04   }$    &  -    &   $ 0.73  _{\pm 0.02   }$  &     -   &  $  1.23 _{\pm  0.01  }$    &    -    &     $  4.42 _{\pm 0.08    }$    &  -      \\

NPTN (18,2) &    $  0.61 _{\pm 0.04   }$   &   -    &   $ 0.63  _{\pm 0.02   }$   &   -    &  $  1.11 _{\pm  0.02  }$     &   -    &     $  4.10 _{\pm 0.08    }$    &  -      \\
NPTN (12,3) &   $ 0.66  _{\pm  0.02  }$    &      -      &   $ 0.64  _{\pm 0.02    }$   &   -    &   $ 1.09  _{\pm 0.04   }$     &    -      &   $4.19   _{\pm 0.04}$          &  - \\
NPTN (9,4) &    $  0.65 _{\pm  0.05  }$   &    -   &   $0.65   _{\pm 0.03    }$    &   -   &  $ 1.16  _{\pm 0.04   }$     &   -    &     $ 4.42  _{\pm 0.07    }$    &     -   \\

\hline
PRC-NPTN (36,1) &    $  0.65 _{\pm 0.02   }$   &      $ 0.65 _{\pm 0.05} $     &  $ 0.58  _{\pm 0.01   }$   &      $ 0.61 _{\pm 0.04 } $    &  $ 1.02  _{\pm 0.03}$     &      $ 1.00 _{\pm 0.04 } $      &   $  3.85 _{\pm  0.11  }$          &   $ 3.83 _{\pm 0.10} $  \\
PRC-NPTN (18,2) &    $   \mathbf{0.59 _{\pm 0.07 }  }$   &    $ 0.59 _{\pm 0.03 } $      &   $  \mathbf{0.52  _{\pm 0.03  } }$   &     $ 0.58 _{\pm 0.02 } $     &  $  \mathbf{0.80  _{\pm 0.03  }  }$     &       $ 0.88 _{\pm 0.05} $     &   $   \mathbf{3.23 _{\pm 0.03  } }$         &    $ 3.34 _{\pm 0.06} $  \\
PRC-NPTN (12,3) &    $  0.63 _{\pm 0.02   }$   &       $ 0.66 _{\pm 0.08 } $    &  $ 0.55  _{\pm 0.02   }$   &     $ 0.59 _{\pm 0.01 } $    &   $ 0.84  _{\pm 0.04    }$     &     $ 0.89 _{\pm 0.03 } $       &   $ 3.35  _{\pm  0.04  }$         &    $ 3.52 _{\pm 0.12 } $  \\

PRC-NPTN (9,4) &    $  0.65 _{\pm 0.02   }$   &    $ 0.69 _{\pm 0.03 } $      &   $ 0.56  _{\pm 0.03   }$   &    $ 0.56 _{\pm 0.03 } $      &  $ 0.88  _{\pm 0.02    }$     &      $0.97  _{\pm 0.02 } $     &    $  3.49 _{\pm  0.46  }$         &   $ 3.69 _{\pm 0.08 } $   \\
\hline

\end{tabular}
\caption{\textbf{Individual Transformation Results: }Test errors on  MNIST with progressively extreme transformations with a) \textbf{random rotations} and b) \textbf{random pixel shifts}.  $***$ indicates ablation runs without any randomization \emph{i.e.} without any random connectomes (applicable only to PRC-NPTNs). For PRC-NPTN and NPTN the brackets indicate the number of channels in the layer 1 and $G$. ConvNet FC denotes the addition of a 2-layered pooling $1\times 1$ pooling network after every layer. Note that for this experiment, CMP=$|G|$. Permanent Random Connectomes help with achieving better generalization despite increased nuisance transformations.}
\label{tab_exp_trans_1_rot}

\end{table*}

\begin{table*}[t]
\tiny
\centering
\begin{tabular}{l | c c c c c c c  } 

\hline\hline 
Rot/Trans   &   $0^\circ$ 0 & $15^\circ$ 2 &  $30^\circ$  4  & $45^\circ$ 6 &     $60^\circ$  8  & $75^\circ$ 10 &  $90^\circ$   12   \\
\hline
\hline

ConvNet (36)  &   $  0.68 _{\pm  0.03    }$  &   $ 0.72  _{\pm 0.02   }$   &   $ 1.31  _{\pm   0.02 }$   &   $ 2.32  _{\pm 0.04   }$    &   $ 5.06  _{\pm 0.04   }$   &  $  10.90 _{\pm 0.08    }$     &     $ 19.60  _{\pm 0.16    }$     \\
ConvNet (36) FC   &   $ 0.64 _{\pm  0.03    }$   &   $ 0.66 _{\pm 0.01     }$   &   $ 0.95 _{\pm    0.04  }$   &   $ 1.50 _{\pm 0.02     }$   &   $ 3.42 _{\pm 0.03     }$   &   $ 8.14 _{\pm 0.11     }$   &   $ 15.61 _{\pm 0.11     }$   \\
ConvNet (512)  &   $  0.66 _{\pm  0.05    }$  &   $ 0.65  _{\pm 0.02   }$   &   $ 0.97  _{\pm   0.02 }$   &   $ 1.60  _{\pm 0.04   }$    &   $ 3.50  _{\pm 0.04   }$   &  $  7.90_{\pm 0.06    }$     &     $ 15.19  _{\pm 0.09    }$     \\


\hline

NPTN (36,1)  &   $  0.71 _{\pm  0.04    }$  &   $ 0.78  _{\pm 0.02   }$   &   $ 1.27  _{\pm   0.02 }$   &   $ 2.35  _{\pm 0.03   }$    &   $ 5.02  _{\pm 0.14   }$   &  $  11.08 _{\pm 0.09    }$     &     $ 19.66  _{\pm 0.33    }$     \\

NPTN (18,2) &   $  0.65 _{\pm  0.02    }$  &   $ 0.68  _{\pm 0.02   }$   &   $ 1.09  _{\pm   0.02 }$   &   $ 1.94  _{\pm 0.04   }$    &   $ 4.17  _{\pm 0.06   }$   &  $  9.59 _{\pm 0.10    }$     &     $ 17.92  _{\pm 0.20    }$     \\

NPTN (12,3) &   $  0.66 _{\pm  0.02    }$  &   $ 0.69  _{\pm 0.03   }$   &   $ 1.07  _{\pm   0.03 }$   &   $ 1.85  _{\pm 0.02   }$    &   $ 4.24  _{\pm 0.11   }$   &  $  9.58 _{\pm 0.06    }$     &     $ 17.79  _{\pm 0.16    }$     \\

NPTN (9,4) &   $  0.64 _{\pm  0.01    }$  &   $ 0.71  _{\pm 0.02   }$   &   $ 1.09  _{\pm   0.04 }$   &   $ 1.98  _{\pm 0.04   }$    &   $ 4.41  _{\pm 0.09   }$   &  $  9.78 _{\pm 0.16    }$     &     $ 18.14  _{\pm 0.16    }$     \\

\hline
PRC-NPTN (36,1)  &   $ 0.61 _{\pm  0.03   }$   &   $ 0.70 _{\pm 0.01    }$  &   $ 1.09 _{\pm 0.04     }$  &   $ 1.80 _{\pm  0.02   }$  &   $ 3.93 _{\pm 0.02    }$  &   $ 9.09 _{\pm 0.11     }$  &   $ 17.03 _{\pm 0.13    }$   \\

PRC-NPTN (18,2)  &   $  \mathbf{0.57 _{\pm  0.02    }}$  &   $\mathbf{ 0.58  _{\pm 0.01 }  }$   &   $ \mathbf{0.77  _{\pm   0.02} }$   &   $\mathbf{ 1.21  _{\pm 0.07 }  }$    &   $\mathbf{ 2.74  _{\pm 0.04}   }$   &  $ \mathbf{ 6.78 _{\pm 0.12 }   }$     &     $ \mathbf{13.79  _{\pm 0.08}    }$      \\

PRC-NPTN (12,3)  &   $  0.59 _{\pm  0.03    }$  &   $ 0.58  _{\pm 0.01   }$   &   $ 0.78  _{\pm   0.02 }$   &   $ 1.26  _{\pm 0.02   }$    &   $ 2.91  _{\pm 0.05   }$   &  $  7.13 _{\pm 0.09    }$     &     $ 14.23  _{\pm 0.07    }$     \\


PRC-NPTN (9,4)  &   $  0.63 _{\pm  0.04    }$  &   $ 0.59  _{\pm 0.02   }$   &   $ 0.81  _{\pm   0.02 }$   &   $ 1.35  _{\pm 0.02   }$    &   $ 3.12  _{\pm 0.02   }$   &  $  7.26 _{\pm 0.02    }$     &     $ 14.62  _{\pm 0.16    }$  \\
\hline

\end{tabular}
\caption{\textbf{Simultaneous Transformation Results: } Test errors on  MNIST with progressively extreme transformations with \textbf{random rotations} and \textbf{random pixel shifts simultaneously}.  For PRC-NPTN and NPTN the brackets indicate the number of channels in the layer 1 and $G$. Note that for this experiment, CMP=$|G|$.}
\label{tab_exp_trans_2_rot}

\end{table*}

\bibliography{egbib}
\end{document}